\documentclass[11pt]{article}

\usepackage{graphicx}
\usepackage{amsfonts}
\usepackage{amsmath,amssymb,amsthm}
\usepackage[ruled]{algorithm2e}
\usepackage{algorithmic}
\usepackage{enumerate}
\usepackage{macros}
\usepackage{fullpage}
\usepackage{color}
\usepackage{epstopdf}
\usepackage[bookmarks=false,colorlinks,citecolor=blue,urlcolor=blue]{hyperref}

\newtheorem{theorem}{Theorem}
\newtheorem{lemma}{Lemma}
\newtheorem{proposition}{Proposition}

\title{Spectral Methods meet EM: A Provably Optimal Algorithm\\ for Crowdsourcing}

\author{
	Yuchen Zhang%
	\thanks{
	Department of Electrical Engineering and Computer Science, University of California at Berkeley, Berkeley, CA 94720. Email: {\tt yuczhang@berkeley.edu}. }
\and
     Xi Chen%
	\thanks{
	Department of Electrical Engineering and Computer Science, University of California at Berkeley, Berkeley, CA 94720. Email: {\tt xichen1987@berkeley.edu}. }
\and
	Dengyong Zhou
	\thanks{Machine Learning Group, Microsoft Research, Redmond, WA 98052. Email: {\tt dengyong.zhou@microsoft.com}.}
\and
    Michael I. Jordan
	\thanks{
	Department of Electrical Engineering and Computer Science,  and Department of Statistics, University of California at Berkeley, Berkeley, CA 94720. Email: {\tt jordan@berkeley.edu}. }
}

\begin{document}
\maketitle
\begin{abstract}
    Crowdsourcing is a popular paradigm for effectively collecting labels at low cost. The Dawid-Skene estimator has been widely used for inferring the true labels from the noisy labels provided by non-expert crowdsourcing workers. However, since the estimator maximizes a non-convex log-likelihood function, it is hard to theoretically justify its performance. In this paper, we propose a two-stage efficient algorithm for multi-class crowd labeling problems. The first stage uses the spectral method to obtain an initial estimate of parameters. Then the second stage refines the estimation by optimizing the objective function of the Dawid-Skene estimator via the EM algorithm. We show that our algorithm achieves the optimal convergence rate up to a logarithmic factor. We conduct extensive experiments on synthetic and real datasets. Experimental results demonstrate that the proposed algorithm is comparable to the most accurate empirical approach, while outperforming several other recently proposed methods.
\end{abstract}

\newpage

\section{Introduction}

With the advent of online services such as Amazon Mechanical Turk, crowdsourcing has become an efficient and inexpensive way to collect labels for large-scale data. Despite the efficiency and immediate availability are virtues of crowdsourcing, labels collected from the crowd can be of low quality since crowdsourcing workers are often non-experts and sometimes unreliable. As a remedy, most crowdsourcing services resort to labeling redundancy, , collecting multiple labels from different workers for each item. Such a strategy raises a fundamental problem in crowdsourcing: how to infer true labels from noisy but redundant worker labels?

For labeling tasks with $k$ different categories,  
Dawid and Skene \cite{dawid1979maximum} develop a maximum likelihood approach based on the EM algorithm.  They assume that each worker is associated with a $k \times k$ confusion matrix, where the $(l, c)$-th entry represents the probability that a random chosen item in  class $l$ is labeled as class $c$ by the worker.  The true labels and worker confusion matrices are jointly estimated by maximizing the likelihood of the observed labels, where the unobserved true labels are treated as latent variables.

Although this EM-based approach has had empirical success~\cite{snow2008cheap, Raykar:10, liu2012variational, zhou2012learning, CheLinZho13, zhouliuplattmeek:14}, there is as yet no theoretical guarantee for its performance. A recent theoretical study \cite{Gao:EM:14} shows that the global optimal solutions of the Dawid-Skene estimator can achieve minimax rates of convergence in a simplified scenario,  where the labeling task is binary and each worker has a single parameter to represent  her labeling accuracy (referred to as ``one-coin" model in what follows). However, since the likelihood function is nonconvex, this guarantee is not operational because the EM algorithm can get trapped in a local optimum.  Several alternative approaches have been developed that aim to circumvent the theoretical deficiencies of the EM algorithm, still the context of the one-coin model~\cite{karger2013efficient, Oh:12, ghosh2011moderates, Dalvi:13}, but, as we survey in Section \ref{sec:related}, they either fail to achieve the optimal rates or make restrictive assumptions which can be hard to justify in practice.

We propose a computationally efficient and provably optimal algorithm to simultaneously estimate true labels and worker confusion matrices for multi-class labeling problems. Our approach is a two-stage procedure, in which we first compute an initial estimate of worker confusion matrices using the spectral method, and then in the second stage we turn to the EM algorithm. Under some mild conditions, we show that this two-stage procedure achieves minimax rates of convergence up to a logarithmic factor, even after only one iteration of EM.
In particular, given any $\delta \in (0,1)$, we provide the bounds on the number of workers and the number of items so that our method can correctly estimate labels for all items with probability at least $1-\delta$. Then we establish the lower bound to demonstrate its optimality. Further, we provide both upper and lower bounds for estimating the confusion matrix of each worker and show that our algorithm achieves the optimal accuracy.

This work not only  provides an optimal algorithm for crowdsourcing but sheds light on understanding the general method of moments. Empirical studies show that when the spectral method is used as an initialization for the EM algorithm, it outperforms EM with random initialization~\cite{Liang:NIPS:13, chaganty2013spectral}. This work provides a concrete way to theoretically justify such observations. It is also known that starting from a root-$n$ consistent estimator obtained by the spectral method, one Newton-Raphson step
leads to an asymptotically optimal estimator~\cite{TOPE:03}.
However, obtaining a root-$n$ consistent estimator and performing a Newton-Raphson step can be demanding computationally.
In contrast, our initialization doesn't need to be root-$n$ consistent, thus a small portion of data suffices to initialize. Moreover, performing one iteration of EM is computationally more attractive and numerically more robust than a Newton-Raphson step especially for high-dimensional problems.

The paper is organized as follows. In Section \ref{sec:related}, we provide background on crowdsourcing and the method of moments for latent variables models. In Section \ref{sec:setup}, we describe our crowdsourcing problem. Our provably optimal algorithm is presented in Section \ref{sec:main-algorithm}. Section \ref{sec:convergence} is devoted to theoretical analysis (with the proofs gathered in the Appendix). In Section \ref{sec:one-coin}, we consider the special case of the one-coin model.  A simpler algorithm is introduced  together with a sharper rate. Numerical results on both synthetic and real datasets are reported in Section~\ref{sec:exp}, followed by our conclusions in Section \ref{sec:conclusion}.

\section{Related Work}
\label{sec:related}

Many methods have been proposed to address the problem of estimating true labels in crowdsourcing
\cite{Whitehill:09, Raykar:10, welinder2010multidimensional, ghosh2011moderates, liu2012variational, zhou2012learning, Dalvi:13,  Oh:12, karger2013efficient, zhouliuplattmeek:14}. The methods in~\cite{Raykar:10, ghosh2011moderates, Oh:12, liu2012variational,  karger2013efficient,Dalvi:13} are based on the generative model proposed by Dawid and Skene~\cite{dawid1979maximum}. In particular, Ghosh et al.~\cite{ghosh2011moderates} propose a method based on Singular Value Decomposition (SVD) which addresses binary labeling problems under the one-coin model. The analysis in~\cite{ghosh2011moderates} assumes that the labeling matrix is full, that is, each worker labels all items. To relax this assumption, Dalvi et al.~\cite{Dalvi:13} propose another SVD-based algorithm which explicitly considers the sparsity of the labeling matrix in both algorithm design and theoretical analysis.
Karger et al.~propose an iterative algorithm for binary labeling problems under the one-coin model~\cite{Oh:12} and extended it to multi-class labeling tasks by converting a $k$-class problem  into $k-1$ binary problems~\cite{karger2013efficient}. This line of work assumes that tasks are assigned to workers according to a random regular graph, thus imposes specific constraints on the number of workers and the number of items.
In Section~\ref{sec:convergence}, we compare our theoretical results with that of existing approaches~\cite{ghosh2011moderates,Dalvi:13,Oh:12,karger2013efficient}. The methods in \cite{Raykar:10,liu2012variational,CheLinZho13} incorporate Bayesian inference into the Dawid-Skene estimator by assuming a prior over confusion matrices.
Zhou et al. \cite{zhou2012learning,zhouliuplattmeek:14} propose a minimax entropy principle for crowdsourcing which leads to an exponential family model  parameterized with worker ability and item difficulty.
When all items have zero difficulty, the exponential family model reduces to the generative model suggested by Dawid and Skene \cite{dawid1979maximum}.

Our method for initializing the EM algorithm in crowdsourcing is inspired by recent work using spectral methods to estimate latent variable models \cite{anandkumar2012tensor,anandkumar2012spectral,Anandkumar:12, Anandkumar:13, chaganty2013spectral,Zou:13, Hsu:12, jain2013learning}. The basic idea in this line of work is to compute third-order empirical moments from the data and then to estimate parameters by computing a certain orthogonal decomposition of tensor derived from the moments. Given the special symmetric structure of the moments, the tensor factorization can be computed efficiently using the robust tensor power method \cite{anandkumar2012tensor}.  A problem with this approach is that the estimation error can have a poor dependence on the condition number of the second-order moment matrix and thus empirically it sometimes performs worse than EM with multiple random initializations.  Our method, by contrast, requires only a rough initialization from the moment of moments; we show that the estimation error does not depend on the condition number (see Theorem \ref{theorem:stage-2} (b)).

\section{Problem Setting}
\label{sec:setup}

Throughout this paper, $[a]$ denotes the integer set $\{1,2,\dots,a\}$ and
$\sigma_b(A)$ denotes the $b$-th largest singular value of matrix $A$.
Suppose that there are $\mdim$ workers, $\numobs$ items and $k$ classes.  The true label $y_j$ of item $j \in [n]$ is assumed to be sampled from a probability distribution
$\mprob[y_j = l] = w_l$ where $\{w_l:l\in[\kdim]\}$ are positive values satisfying $\sum_{l=1}^\kdim w_l = 1$.
Denote by a vector $z_{ij}\in \R^\kdim$  the label that worker $i$ assigns to item $j$. When the assigned label is $c, $ we write $z_{ij} = e_c$, where $e_c$ represents the $c$-th canonical basis vector in $\R^\kdim  $ in which the $c$-th entry is $1$ and all other entries are $0.$  A worker may not label every item.
Let  $\pi_i$ indicate the probability that worker $i$ labels a randomly chosen item.
If item $j$  is not labeled by worker $i$, we write $z_{ij} = 0$.
Our goal is to estimate the  true labels $\{y_j: j \in [n]\}$ from the observed labels $\{z_{ij}: i \in [m], j \in [n]\}.$

For this estimation purpose, we need to make assumptions on the process of generating observed labels. Following the work of  Dawid and Skene \cite{dawid1979maximum},  we assume  that the probability that worker $i$ labels an item in class $l$ as class $c$ is independent of any particular chosen item,  that is, it is a constant over $j \in [n]$.   Let us denote the constant probability by $\mu_{ilc}.$
Let $\mu_{il} = [\mu_{il1}~\mu_{il2}~\cdots~\mu_{ilk}]^T.$ The matrix $
  \cfmatrix_i = [\mu_{i1}~\mu_{i2}~\dots~\mu_{i\kdim}] \in \R^{\kdim\times\kdim} $
 is called the \emph{confusion matrix} of worker $i$. In the special case of the one-coin model, all the diagonal elements of $C_i$ are equal to a constant while all the off-diagonal elements are equal to another constant such that each row of $C_i$ sums to $1.$

\section{Our Algorithm}
\label{sec:main-algorithm}

\begin{algorithm}[t]
\DontPrintSemicolon
\KwIn{integer $k$, observed labels $z_{ij}\in\R^\kdim$ for $i\in[\mdim]$ and $j\in [\numobs]$.}
\KwOut{confusion matrix estimates $\Chat_i\in \R^{\kdim\times\kdim}$ for $i\in[\mdim]$.}
\begin{enumerate}[(1)]
  \item Partition the workers into three disjoint and non-empty group $G_1$, $G_2$ and $G_3$. Compute
  the group aggregated labels $Z_{gj}$ by Eq.~\eqref{eqn:group-agregated-label}.
  \item For $(a,b,c)\in \{(2,3,1),(3,1,2),(1,2,3)\}$, compute the second and the third order moments
  $\Mhat_2\in\R^{\kdim\times\kdim}$, $\Mhat_3\in\R^{\kdim\times\kdim\times \kdim}$ by Eq.~\eqref{eqn:compute-moments-1}-\eqref{eqn:compute-moments-4},
  then compute $\Cdiamhat_c\in \R^{k\times k}$ and $\What\in \R^{k\times k}$ by tensor decomposition:
  \begin{enumerate}[(a)]
    \item Compute whitening matrix $\Qhat \in\R^{\kdim\times\kdim}$ (such that $\Qhat^T \Mhat_2 \Qhat = I$) using SVD.
    \item Compute eigenvalue-eigenvector pairs $\{(\alphahat_h,\vhat_h)\}_{h=1}^\kdim$ of the whitened tensor $\Mhat_3(\Qhat,\Qhat,\Qhat)$ by using
    the robust tensor power method. Then compute $\weighthat_h = \alphahat_h^{-2}$ and $\mudiamhat_h = (\Qhat^T)^{-1}(\alphahat_h\vhat_h)$.
    \item For $l=1,\dots,\kdim$, set the $l$-th column of $\Cdiamhat_c$ by some $\mudiamhat_h$ whose $l$-th coordinate has the greatest component,
    then set the $l$-th diagonal entry of $\What$ by $\weighthat_h$.
  \end{enumerate}
  \item Compute $\Chat_i$ by Eq.~\eqref{eqn:compute-confusion-matrix-estimate}.
\end{enumerate}
\caption{Estimating confusion matrices}
\label{alg:estimating-confusion-matrix}
\end{algorithm}

In this section, we present an algorithm to estimate the confusion matrices and true labels. Our algorithm consists of two stages. In the first stage, we compute an initial
estimate for the confusion matrices via the method of moments. In the second stage,
we perform the standard EM algorithm by taking the result of the Stage 1 as an initialization.

\subsection{Stage 1: Estimating Confusion Matrices}

Partitioning the workers into three disjoint and non-empty groups $G_1$, $G_2$ and $G_3$,
the outline of this stage is the following: we use the method of moments to estimate the averaged confusion
matrices for the three groups, then utilize this intermediate estimate to obtain the confusion matrix
of each individual worker. In particular, for $g\in\{1,2,3\}$ and $j\in [\numobs]$, we calculate
the averaged labeling within each group by
\begin{align}\label{eqn:group-agregated-label}
  Z_{gj} \defeq \frac{1}{|G_g|}\sum_{i\in G_{g}} z_{ij}.
\end{align}
Denoting the aggregated confusion matrix columns by
\begin{align*}
\mudiam_{gl} \defeq \E(Z_{gj} | y_j =l) =  \frac{1}{|G_g|} \sum_{i\in G_{g}} \pi_i \mu_{il},
\end{align*}
our first step is to estimate $\Cdiam_g \defeq [\mudiam_{g1},\mudiam_{g2},\dots,\mudiam_{g\kdim}]$ and to estimate
the distribution of true labels $W\defeq {\rm diag}(w_1,w_2,\dots,w_\kdim)$.
The following proposition shows that we can solve for $\Cdiam_g$ and $W$ from the moments of $\{Z_{gj}\}$.

\begin{proposition}[Anandkumar et al.~\cite{anandkumar2012spectral}]\label{prop:multi-view-model}
  Assume that the vectors $\{\mudiam_{g1},\mudiam_{g2},\dots,\mudiam_{g\kdim}\}$ are linearly independent for each $g\in\{1,2,3\}$. Let
  $(a,b,c)$ be a permutation of $\{1,2,3\}$. Define
  \begin{align*}
    Z'_{aj} &\defeq \E[Z_{cj}\otimes Z_{bj}]\left( \E[Z_{aj}\otimes Z_{bj}] \right)^{-1} Z_{aj},\\
    Z'_{bj} &\defeq \E[Z_{cj}\otimes Z_{aj}]\left( \E[Z_{bj}\otimes Z_{aj}] \right)^{-1} Z_{bj},\\
    M_2 &\defeq \E[Z'_{aj}\otimes Z'_{bj}],\\
    M_3 &\defeq \E[Z'_{aj}\otimes Z'_{bj}\otimes Z_{cj}],
  \end{align*}
  Then,
  \begin{align*}
    M_2 = \sum_{l=1}^k w_l\; \mudiam_{cl}\otimes \mudiam_{cl}\quad \mbox{and} \quad
		M_3 = \sum_{l=1}^k w_l\; \mudiam_{cl}\otimes \mudiam_{cl}\otimes \mudiam_{cl}.\\
  \end{align*}
\end{proposition}

Since we only have finite samples, the expectations in Proposition~\ref{prop:multi-view-model}
must be approximated by empirical moments.
In particular, they are computed by averaging over indices $j = 1,2,\dots,\numobs$. For each permutation $(a,b,c) \in \{(2,3,1),(3,1,2),(1,2,3)\}$, we compute
\begin{subequations}
\begin{align}
    \Zhat'_{aj} &\defeq \Big(\frac{1}{\numobs}\sum_{j=1}^\numobs Z_{cj}\otimes Z_{bj} \Big)\Big(\frac{1}{\numobs}\sum_{j=1}^\numobs Z_{aj}\otimes Z_{bj}\Big)^{-1} Z_{aj},\label{eqn:compute-moments-1}\\
    \Zhat'_{bj} &\defeq \Big(\frac{1}{\numobs}\sum_{j=1}^\numobs Z_{cj}\otimes Z_{aj} \Big)\Big(\frac{1}{\numobs}\sum_{j=1}^\numobs Z_{bj}\otimes Z_{aj}\Big)^{-1} Z_{bj},\\
    \Mhat_2 &\defeq \frac{1}{\numobs}\sum_{j=1}^\numobs \Zhat'_{aj} \otimes \Zhat'_{bj},\label{eqn:compute-moments-3}\\
    \Mhat_3 &\defeq \frac{1}{\numobs}\sum_{j=1}^\numobs \Zhat'_{aj} \otimes \Zhat'_{bj} \otimes Z_{cj}.\label{eqn:compute-moments-4}
\end{align}
\end{subequations}

The statement of Proposition~\ref{prop:multi-view-model}
suggests that we can recover the columns of $\Cdiam_c$ and the diagonal entries of $W$
by operating on the moments $\Mhat_2$ and $\Mhat_3$. This is implemented
by the  the tensor factorization method in Algorithm~\ref{alg:estimating-confusion-matrix}.
In particular, the tensor factorization algorithm returns a set of vectors $\{(\mudiamhat_h,\weighthat_h): h=1,\dots,\kdim\}$,
where each $(\mudiamhat_h,\weighthat_h)$ estimates a particular column of $\Cdiam_{c}$ (for some $\mudiam_{cl}$) and a particular diagonal entry of $W$ (for some $w_l$).
It is important to note that the tensor factorization algorithm doesn't provide a one-to-one correspondence
between the recovered column and the true columns of $\Cdiam_{c}$. Thus, $\mudiamhat_1,\dots,\mudiamhat_\kdim$
represents an arbitrary permutation of the true columns.

To discover the index correspondence,
we take each $\mudiamhat_h$ and examine its greatest component.
We assume that within each group, the probability of assigning a
correct label is always greater than the probability of assigning any specific incorrect label.
This assumption will be made precise in the next section. As a consequence,
if $\mudiamhat_h$ corresponds to the $l$-th column of $\Cdiam_c$,
then its $l$-th coordinate is expected to be greater than other coordinates.
Thus, we set the $l$-th column of $\Cdiamhat_c$ to some vector $\mudiamhat_h$
whose $l$-th coordinate has the greatest component (if there are multiple such vectors,
then randomly select one of them; if there is no such vector,
then randomly select a $\mudiamhat_h$). Then, we set the $l$-th diagonal entry of $\What$
to the scalar $\weighthat_h$ associated with $\mudiamhat_h$.
Note that by iterating over $(a,b,c) \in \{(2,3,1),(3,1,2),(1,2,3)\}$, we
obtain $\Cdiamhat_c$ for $c=1,2,3$ respectively. There will be three copies
of $\What$ estimating the same matrix $W$---we average them for the best accuracy.\\

In the second step, we estimate each individual confusion matrix $C_i$. The following proposition
shows that we can recover $C_i$ from the moments of $\{z_{ij}\}$.

\begin{proposition}\label{prop:recover-individual-matrix}
  For any $g\in\{1,2,3\}$ and any $i\in G_g$, let $a \in \{1,2,3\}\backslash\{g\}$ be one of the remaining group index. Then
  \begin{align*}
    \pi_i C_i W (\Cdiam_a)^T = \E[z_{ij}Z_{aj}^T].\\
  \end{align*}
\end{proposition}
\begin{proof}
 First, notice that
 \begin{align}
     \E[z_{ij}Z_{aj}^T]= \E\left[\E[z_{ij}Z_{aj}^T |y_j]\right]= \sum_{l=1}^k w_l \E\left[z_{ij}Z_{aj}^T |y_j=l\right].
     \label{eqn:recover_indiv_1}
 \end{align}
  Since $z_{ij}$ for $1 \leq i \leq m$ are conditionally independent given $y_j$, we can write
  \begin{align}
    \E\left[z_{ij}Z_{aj}^T |y_j=l\right]=     \E\left[z_{ij} |y_j=l \right] \E\left[Z_{aj}^T |y_j=l\right]=(\pi_i \mu_{il})(\mudiam_{al})^T.
         \label{eqn:recover_indiv_2}
  \end{align}
  Combining \eqref{eqn:recover_indiv_1} and \eqref{eqn:recover_indiv_2} implies the desired result,
  \begin{align*}
    \E[z_{ij}Z_{aj}^T]= \pi_i \sum_{l=1}^k w_l \mu_{il}(\mudiam_{al})^T=\pi_i C_i W (\Cdiam_a)^T.
  \end{align*}
\end{proof}

Proposition~\ref{prop:recover-individual-matrix} suggests a plug-in estimator for $C_i$. We compute $\Chat_i$
using the empirical approximation of $\E[z_{ij}Z_{aj}^T]$ and using the matrices
$\Cdiamhat_a$, $\Cdiamhat_b$, $\What$ obtained in the first step. Concretely, we calculate
\begin{align}\label{eqn:compute-confusion-matrix-estimate}
  \Chat_i \defeq {\rm normalize}\left\{ \Big(\frac{1}{\numobs}\sum_{j=1}^\numobs z_{ij}Z_{aj}^T \Big)\Big( \What(\Cdiamhat_a)^T \Big)^{-1} \right\},
\end{align}
where the normalization operator rescales the matrix columns, making sure that each column sums to $1$.
The overall procedure for Stage 1 is summarized in Algorithm~\ref{alg:estimating-confusion-matrix}.

\subsection{Stage 2: EM algorithm}

The second stage is devoted to refining the initial estimate provided by
Stage 1.  The joint likelihood of true label $y_j$ and observed labels $z_{ij}$,
as a function of confusion matrices $\mu_{i}$, can be written as
\begin{align*}
  L(\mu; y,z) &\defeq \prod_{j=1}^\numobs \prod_{i=1}^\mdim \prod_{c=1}^\kdim (\mu_{iy_jc})^{\indicator(z_{ij} = e_c)}.
\end{align*}
By assuming a uniform prior over $y$, we maximize the marginal log-likelihood function
\begin{align}
  \ell(\mu) \defeq \log\left( \sum_{y\in [\kdim]^\numobs} L(\mu; y,z)\right).\label{eqn:likelihood-lower-bound}
\end{align}

We refine the initial estimate of Stage 1 by maximizing the objective function~\eqref{eqn:likelihood-lower-bound},
which is implemented by the Expectation Maximization (EM)
algorithm. The EM algorithm takes as initialization the values
$\{\muhat_{ilc}\}$ provided as output by Stage 1, and
then executes the following E-step and M-step for at least one round.
\begin{description}
  \item[E-step] Calculate the expected value of the log-likelihood function, with respect to the conditional distribution of $y$ given $z$ under the current estimate of $\mu$:
  \begin{align}
    & Q(\mu) \defeq \E_{y|z,\muhat}\left[ \log(L(\mu; y,z)) \right] = \sum_{j=1}^\numobs \left\{ \sum_{l=1}^\kdim \qhat_{jl}\log\left( \prod_{i=1}^\mdim \prod_{c=1}^\kdim (\mu_{ilc})^{\indicator(z_{ij} = e_c)} \right) \right\},\nonumber\\
    &\mbox{where}\quad  \qhat_{jl} \leftarrow \frac{\exp\big( \sum_{i=1}^\mdim \sum_{c=1}^\kdim \indicator(z_{ij} = e_c)\log(\muhat_{ilc}) \big)}
  {\sum_{l'=1}^\kdim \exp\big( \sum_{i=1}^\mdim \sum_{c=1}^\kdim \indicator(z_{ij} = e_c)\log(\muhat_{il'c}) \big)} \qquad \mbox{for $j\in[\numobs]$, $l\in[\kdim]$.}\label{eqn:q_update}
  \end{align}
  \item[M-step] Find the estimate $\muhat$ that maximizes the function $Q(\mu)$:
  \begin{align}
    \muhat_{ilc} &\leftarrow \frac{ \sum_{j=1}^n \qhat_{jl} \indicator(z_{ij}=e_c) }{ \sum_{c'=1}^\kdim \sum_{j=1}^n \qhat_{jl} \indicator(z_{ij}=e_{c'}) } \qquad \mbox{for $i\in[\mdim]$, $l\in[\kdim]$, $c\in [\kdim]$.} \label{eqn:p_update}
  \end{align}
\end{description}
In practice, we alternatively execute the updates~\eqref{eqn:q_update} and \eqref{eqn:p_update},
for one iteration or until convergence.  Each update increases the objective function~$\ell(\mu)$.
Since $\ell(\mu)$ is not concave, the EM update doesn't guarantee converging to the global
maximum. It may converge to distinct local stationary points for different initializations.
Nevertheless, as we prove in the next section, it is guaranteed that the EM algorithm will output statistically optimal estimates of
true labels and worker confusion matrices if it is initialized by Algorithm~\ref{alg:estimating-confusion-matrix}.

\section{Convergence Analysis}
\label{sec:convergence}

To state our main theoretical results, we first need to introduce some notation and assumptions.
Let
\begin{align*}
  \wmin \defeq \min \{w_l\}_{l=1}^\kdim \quad \mbox{and} \quad \pimin \defeq \{\pi_i\}_{i=1}^\mdim
\end{align*}
be the smallest portion of true labels and the most extreme sparsity level of workers. Our first
assumption assumes that both $\wmin$ and $\pimin$ are strictly positive, that is,  every class
and every worker contributes to the dataset.

Our second assumption assumes that the confusion matrices for each of the three groups, namely $\Cdiam_1$, $\Cdiam_2$ and $\Cdiam_3$,
are nonsingular. As a consequence, if we define matrices $S_{ab}$ and tensors $T_{abc}$ for any $a,b,c\in\{1,2,3\}$ as
\begin{align*}
  S_{ab} \defeq \sum_{l=1}^\kdim w_l\;\mudiam_{al}\otimes \mudiam_{bl} = \Cdiam_a W (\Cdiam_b)^T \quad \mbox{and} \quad T_{abc} &\defeq \sum_{l=1}^\kdim w_l\; \mudiam_{al}\otimes \mudiam_{bl}\otimes \mudiam_{cl},
\end{align*}
then there will be a positive scalar $\sigmadown$ such that $\sigma_\kdim(S_{ab}) \geq \sigmadown > 0$.

Our third assumption assumes that within each group, the average probability of assigning
a correct label is always higher than the average probability of assigning any incorrect label.
To make this statement rigorous, we define a quantity
\begin{align*}
  \dgap \defeq \min_{g\in\{1,2,3\}} \min_{l\in[\kdim]} \min_{c\in [\kdim]\backslash\{l\}}\{ \mudiam_{gll} - \mudiam_{glc} \}
\end{align*}
indicating the smallest gap between diagonal entries and non-diagonal entries in the confusion matrix.
The assumption requires that $\dgap$ is strictly positive. Note that this assumption is group-based,
thus doesn't assume the accuracy of any individual worker.

Finally, we introduce a quantity that measures the average ability of workers
in identifying distinct labels. For two discrete distributions $P$ and $Q$, let
$ \KL{P,Q} \defeq \sum_i P(i)\log(P(i)/Q(i)) $
represent the KL-divergence between $P$ and $Q$. Since each column of the confusion
matrix represents a discrete distribution, we can define the following quantity:
\begin{align}\label{eqn:define-min-kl-distance}
  \Dbar = \min_{l\neq l'} \frac{1}{\mdim}\sum_{i=1}^\mdim \pi_i \KL{\mu_{il},\mu_{il'}}.
\end{align}
The quantity $\Dbar$ lower bounds the averaged KL-divergence between two columns. If $\Dbar$
is strictly positive, it means that every pair of labels can be distinguished by at least
one subset of workers. As the last assumption, we assume that $\Dbar$ is strictly positive.
\\

The following two theorems characterize the performance of our algorithm. We
split the convergence analysis into two parts. Theorem~\ref{theorem:stage-1} characterizes the performance
of Algorithm~\ref{alg:estimating-confusion-matrix}, providing sufficient conditions for achieving an arbitrarily
accurate initialization. We provide the proof of Theorem~\ref{theorem:stage-1} in
Appendix~\ref{sec:proof-stage-1}.\\

\begin{theorem}\label{theorem:stage-1}
For any scalar $\delta > 0$ and any scalar $\little$ satisfying $\little\leq \min\left\{ \frac{36\dgap\kdim}{\pimin\wmin \sigmadown},2 \right\}$,
if the number of items $\numobs$ satisfies
\begin{align*}
   \numobs = \Omega\left( \frac{\kdim^{5}\log((\kdim+\mdim)/\delta) }{\little^2 \pimin^2 \wmin^{2}\sigmadown^{13}}  \right),
\end{align*}
then the confusion matrices returned by Algorithm~\ref{alg:estimating-confusion-matrix} are bounded as
\begin{align*}
  \| \Chat_i - C_i\|_\infty \leq \little \qquad \mbox{for all $i\in [\mdim]$},
\end{align*}
with probability at least $1 - \delta$. Here, $\|\cdot\|_\infty$ denotes the element-wise $\ell_\infty$-norm of a matrix.
\end{theorem}
\vspace{10pt}

Theorem~\ref{theorem:stage-2} characterizes the error rate in Stage 2. It states that when a sufficiently accurate
initialization is taken, the updates~\eqref{eqn:q_update} and~\eqref{eqn:p_update} refine the estimates $\muhat$ and $\yhat$
to the optimal accuracy.  See Appendix~\ref{sec:proof-stage-2} for the proof.\\

\begin{theorem}\label{theorem:stage-2}
  Assume that $\mu_{ilc} \geq \gap$ holds for all $(i,l,c)\in [\mdim]\times[\kdim]^2$.
  For any scalar $\delta > 0$, if confusion matrices $\Chat_i$ are initialized in a way such that
  \begin{align}\label{eqn:initialization-accuracy-condition}
    \norms{\Chat_i - C_i}_\infty \leq \minac \defeq \min\left\{ \frac{\gap}{2},\frac{\gap\Dbar}{16}\right\}\qquad \mbox{for all $i\in [\mdim]$}
  \end{align}
  and the number of workers $\mdim$ and the number of items $\numobs$ satisfy
  \begin{align*}
    \mdim = \Omega\left( \frac{\log(1/\gap)\log(\kdim\numobs/\delta)+\log(\mdim\numobs)}{\Dbar} \right)\quad\mbox{and} \quad
    \numobs = \Omega\left( \frac{\log(\mdim\kdim/\delta)}{\pimin\wmin\minac^2} \right),
  \end{align*}
  then, for $\muhat$ and $\qhat$ obtained by iterating~\eqref{eqn:q_update} and~\eqref{eqn:p_update} (for at least one round), with probability
  at least $1-\delta$,
  \begin{enumerate}[(a)]
    \item Let $\yhat_j = \arg\max_{l\in[\kdim]} \qhat_{jl}$, then $\yhat_j = y_j$ holds for all $j\in[\numobs]$.
    \item $\ltwos{ \muhat_{il} - \mu_{il} }^2 \leq \frac{48\log(2 \mdim\kdim/\delta)}{\pi_i w_l \numobs}$ holds for all $(i,l)\in [\mdim]\times[\kdim]$.
  \end{enumerate}
\end{theorem}
\vspace{10pt}

In Theorem~\ref{theorem:stage-2}, the assumption that all confusion matrix entries are lower bounded by $\gap > 0$
is somewhat restrictive. For datasets violating this assumption, we enforce positive confusion matrix entries
by adding random noise: Given any observed label $z_{ij}$,
we replace it by a random label in $\{1,...,k\}$ with probability $k \rho$. In this modified model,
every entry of the confusion matrix is lower bounded by $\rho$, so that Theorem~\ref{theorem:stage-2} holds.
The random noise makes the constant $\Dbar$ smaller than its original value,
but the change is minor for small $\rho$.

To see the consequence of the convergence analysis, we take error rate $\little$ in Theorem~\ref{theorem:stage-1}
equal to the constant~$\minac$ defined in Theorem~\ref{theorem:stage-2}. Then we combine the statements of the two
theorems. This shows that if we choose the number of workers $\mdim$ and the number of items $\numobs$ such
that
\begin{align}\label{eqn:mn-simplified-condition}
  \mdim = \widetilde \Omega\left( \frac{1}{\Dbar}\right) \quad\mbox{and}\quad  \numobs = \widetilde \Omega\left( \frac{\kdim^{5} }{\pimin^2 \wmin^{2}\sigmadown^{13} \min\{\gap^2, (\gap\Dbar)^2\}}  \right);
\end{align}
that is, if both $\mdim$ and $\numobs$ are lower bounded by a problem-specific constant and logarithmic terms,
then with high probability, the predictor $\yhat$ will be perfectly accurate, and the estimator $\muhat$ will be
bounded as $\ltwos{ \muhat_{il} - \mu_{il} }^2 \leq \widetilde \order(1/(\pi_i w_l \numobs))$.
To show the optimality of this convergence rate, we present the following minimax lower bounds. See
Appendix~\ref{sec:proof-optimality} for the proof.\\

\begin{theorem}\label{theorem:optimality}
There are universal constants $\UNICON_1 > 0$ and $\UNICON_2 > 0$ such that:
\begin{enumerate}[(a)]
  \item For any $\{\mu_{ilc}\}$, $\{\pi_{i}\}$ and any number of items~$\numobs$, if the number of workers $\mdim \leq 1/(4\Dbar)$, then
  \begin{align*}
    \inf_{\yhat} \sup_{v\in [\kdim]^\numobs} \E\Big[\sum_{j=1}^\numobs \indicator(\yhat_j\neq y_j) \Big | \{\mu_{ilc}\},\{\pi_{i}\},y = v\Big] \geq \UNICON_1 \numobs.
  \end{align*}

  \item For any $\{w_{l}\}$, $\{\pi_{i}\}$, any worker-item pair $(\mdim,\numobs)$ and any pair of indices $(i,l)\in[\mdim]\times[\kdim]$, we have
  \begin{align*}
    \inf_{\muhat} \sup_{\mu\in \R^{\mdim\times\kdim\times\kdim}} \E\Big[ \ltwos{\muhat_{il} - \mu_{il}}^2 \Big| \{w_l\}, \{\pi_i\} \Big] \geq \UNICON_2\;\min\left\{1, \frac{1}{\pi_{i} w_{l}\numobs}\right\}.
  \end{align*}
\end{enumerate}
\end{theorem}

In part (a) of Theorem~\ref{theorem:optimality},
we see that the number of workers should be at least $1/\Dbar$, otherwise any predictor will make many mistakes.
This lower bound matches our sufficient condition on the number of workers~$\mdim$ (see Eq.~\eqref{eqn:mn-simplified-condition}).
In part (b), we see that the best possible estimate for $\mu_{il}$ has $1/(\pi_i w_l \numobs)$ mean-squared error.
It verifies the optimality of our estimator $\muhat_{il}$. It is also worth noting that the constraint
on the number of items~$\numobs$ (see Eq.~\eqref{eqn:mn-simplified-condition}) depends on problem-specific constants,
which might be improvable.
Nevertheless, the constraint scales  logarithmically with $\mdim$ and $1/\delta$, thus is easy to satisfy for reasonably large datasets.

It is worth contrasting our convergence rate with existing algorithms. Ghosh et al.~\cite{ghosh2011moderates} and
Dalvi et al.~\cite{Dalvi:13} propose consistent estimators for the binary one-coin model. To attain an error
rate~$\delta$, their algorithms require $\mdim$ and $\numobs$ scaling with $1/\delta^2$, while our algorithm
only requires $\mdim$ and $\numobs$ scaling with $\log(1/\delta)$. Karger et al. \cite{Oh:12,karger2013efficient}
propose algorithms for both binary and multi-class problems. Their algorithm assumes that workers are assigned by a random regular graph.
Their analysis assumes that the limit of number of items goes to infinity, or that the number of workers
is many times of the number of items. Our algorithm no longer requires these assumptions.

We also compare our algorithm with the majority voting estimator, where the true label
is simply estimated by a majority vote among workers. Gao and Zhou~\cite{Gao:EM:14}
show that if there are many spammers and few experts, the majority voting estimator gives almost a random guess.
In contrast, our algorithm requires a relatively large $\mdim\Dbar$ to guarantee good performance. Since $\mdim\Dbar$
is the aggregated KL-divergence, a small number of experts are sufficient to ensure it large enough.

\section{One-Coin Model}
\label{sec:one-coin}

In this section, we consider a simpler crowdsourcing model that is usually referred to as the ``one-coin model.''
For the one-coin model, the confusion matrix $C_i$ is parameterized by a single
parameter $p_i$. More concretely, its entries are defined as
\begin{align}\label{eqn:define-onecoin-confusion-matrix}
  \mu_{ilc} = \left\{
  \begin{array}{ll}
    p_i & \mbox{if $l=c$},\\
    \frac{1 - p_i}{\kdim - 1} & \mbox{if $l\neq c$}.
  \end{array}
  \right.
\end{align}
In other words, the worker $i$ uses a single coin flip to decide her assignment. No matter what the true label is,
the worker has $p_i$ probability to assign the correct label,
and has $1-p_i$ probability to randomly assign an incorrect label. For the one-coin model, it suffices to estimate
$p_i$ for every worker $i$ and estimate $y_j$ for every item $j$. Because of its
simplicity, the one-coin model is easier to estimate and enjoys
better convergence properties.\\

To simplify our presentation, we consider the case where $\pi_i \equiv 1$; noting that with
proper normalization, the algorithm can be easily adapted to the case where $\pi_i < 1$.
The statement of the algorithm relies on the following notation:
For every two workers $a$ and $b$, let the quantity $N_{ab}$ be defined as
\begin{align*}
  N_{ab} \defeq \frac{\kdim - 1}{\kdim}\left( \frac{\sum_{j=1}^\numobs \indicator(z_{aj} = z_{bj})}{\numobs} - \frac{1}{\kdim} \right).
\end{align*}
For every worker $i$, let workers $a_i,b_i$ be defined as
\begin{align*}
  (a_i,b_i) = \arg\max_{(a,b)} \{|N_{ab}|:~a\neq b\neq i\}.
\end{align*}
The algorithm contains two separate stages. First, we initializes $\phat_i$ by an estimator
based on the method of moments.  In contrast with the algorithm for the general model, the
estimator for the one-coin model doesn't need third-order moments.
Instead, it only relies on pairwise statistics~$N_{ab}$. Second, an EM algorithm
is employed to iteratively maximize the objective function~\eqref{eqn:likelihood-lower-bound}.
See Algorithm~\ref{alg:one-coin-model} for a detailed description. \\

\begin{algorithm}[t]
\DontPrintSemicolon
\KwIn{integer $k$, observed labels $z_{ij}\in\R^\kdim$ for $i\in[\mdim]$ and $j\in [\numobs]$.}
\KwOut{Estimator $\phat_i$ for $i\in[\mdim]$ and $\yhat_j$ for $j\in [\numobs]$.}
\begin{enumerate}[(1)]
\item Initialize $\phat_i$ by
\begin{align}\label{eqn:onecoin-p-init}
\phat_i \leftarrow \frac{1}{\kdim} + \sign(N_{i a_1})\;\sqrt{\frac{N_{i a_i}N_{i b_i}}{N_{a_i b_i}}}
\end{align}
\item If $\frac{1}{\mdim} \sum_{i=1}^\mdim \phat_i \geq \frac{1}{\kdim}$ not hold, then set $\phat_i \leftarrow \frac{2}{\kdim} - \phat_i$ for all $i \in [\mdim]$.
\item Iteratively execute the following two steps for at least one round:
\begin{align}
  \qhat_{jl} & \propto \exp\Big( \sum_{i=1}^\mdim \indicator(z_{ij} = e_l)\log(\phat_i) + \indicator(z_{ij} \neq e_l)\log\Big( \frac{1-\phat_i}{\kdim - 1} \Big) \Big)
  \qquad \mbox{for $j\in[\numobs]$, $l\in[\kdim]$,} \label{eqn:onecoin-q_update} \\
  \phat_i &\leftarrow \frac{1}{\numobs} \sum_{j=1}^n \sum_{l=1}^\kdim \qhat_{jl}\indicator(z_{ij}=e_l) \qquad \mbox{for $i\in[\mdim]$,} \label{eqn:onecoin-p_update}
\end{align}
where update~\eqref{eqn:onecoin-q_update} normalizes $\qhat_{jl}$, making $\sum_{l=1}^\kdim \qhat_{jl} = 1$ holds for all $j\in[\numobs]$.

\item Output $\{\phat_i\}$ and $\yhat_j \defeq \arg\max_{l\in [\kdim]} \{\qhat_{jl}\}$.
\end{enumerate}
\caption{Estimating one-coin model}
\label{alg:one-coin-model}
\end{algorithm}

To theoretically characterize the performance of Algorithm~\ref{alg:one-coin-model}, we need some
additional notation.  Let $\dgap_i$ be the $i$-th largest element in $\{ |p_i - 1/\kdim|\}_{i=1}^\mdim$.
In addition, let $\dgapbar \defeq \frac{1}{\mdim}\sum_{i=1}^\mdim (p_i - 1/\kdim)$ be the average gap between  all accuracies
and $1/\kdim$. We assume that $\dgapbar$ is strictly positive.
We follow the definition of $\Dbar$ in Eq.~\eqref{eqn:define-min-kl-distance}.
The following theorem is proved in Appendix~\ref{sec:proof-one-coin-model}.

\begin{theorem}\label{theorem:one-coin-model}
  Assume that $\gap \leq  p_i\leq 1-\gap$ holds for all $i\in [\mdim]$.
  For any scalar $\delta > 0$, if the number of workers $\mdim$ and the number of items $\numobs$ satisfy
  \begin{align}\label{eqn:mn-onecoin-condition}
    \mdim = \Omega\left( \frac{\log(1/\gap)\log(\kdim\numobs/\delta)+\log(\mdim\numobs)}{\Dbar} \right)\quad\mbox{and} \quad
    \numobs = \Omega\left( \frac{\log(\mdim\kdim/\delta)}{ \dgap_3^6 \min\{\dgapbar^2, \gap^2, (\gap\Dbar)^2\} } \right),
  \end{align}
  Then, for $\phat$ and $\yhat$ returned by Algorithm~\ref{alg:one-coin-model}, with probability
  at least $1-\delta$,
  \begin{enumerate}[(a)]
    \item $\yhat_j = y_j$ holds for all $j\in[\numobs]$.
    \item $|\phat_i - p_i| \leq 2\sqrt{\frac{3\log(6 \mdim /\delta)}{\numobs}}$ holds for all $i \in [\mdim]$.
  \end{enumerate}
\end{theorem}

It is worth contrasting condition~\eqref{eqn:mn-simplified-condition} with condition~\eqref{eqn:mn-onecoin-condition},
namely the sufficient conditions for the general model and for the one-coin model. It turns out
that the one-coin model requires much milder conditions on the number of items. In particular, $\dgap_3$ will be close to
$1$ if among all the workers there are three experts giving high-quality answers. As a consequence, the one-coin
is more robust than the general model. By contrasting the convergence rate
of $\muhat_{il}$ (by Theorem~\ref{theorem:stage-2}) and $\phat_i$ (by Theorem~\ref{theorem:one-coin-model}), the convergence rate of $\phat_i$ does not depend on $\{w_l\}_{l=1}^\kdim$.
This is another evidence that the one-coin model enjoys a better convergence rate because of its simplicity.

\section{Experiments}
\label{sec:exp}

In this section, we report the results of empirical studies comparing the algorithm we propose in
Section~\ref{sec:main-algorithm} (referred to as Opt-D\&S) with a variety of other methods.
We compare to the Dawid \& Skene estimator initialized by majority voting (refereed to as MV-D\&S),
the pure majority voting estimator, the multi-class labeling algorithm proposed by Karger
et al.~\cite{karger2013efficient} (referred to as KOS), the SVD-based algorithm proposed by
Ghosh et al.~\cite{ghosh2011moderates} (referred to as Ghost-SVD) and the ``Eigenvalues of Ratio''
algorithm proposed by Dalvi et al. \cite{Dalvi:13} (referred to as EigenRatio). The evaluation
is made on three synthetic datasets and five real datasets.

\subsection{Synthetic data}

\begin{table}
\centering
\begin{tabular}{|c|c|c|c|c|c|c|}
  \hline
    & Opt-D\&S & MV-D\&S & Majority Voting & KOS & Ghosh-SVD & EigenRatio\\\hline
  $\pi = 0.2$ & 7.64 & 7.65 & 18.85 & 8.34 & 12.35 & 10.49\\\hline
  $\pi = 0.5$ & 0.84 & 0.84 & 7.97 & 1.04 & 4.52 & 4.52 \\\hline
  $\pi = 1.0$ & 0.01 & 0.01 & 1.57 & 0.02 & 0.15 & 0.15 \\
  \hline
\end{tabular}
\caption{Prediction error (\%) on the synthetic dataset. The parameter $\pi$ indicates the sparsity of data --- it is
the probability that the worker labels each task.}\label{tab:synthetic-prediction-error}
\end{table}

\begin{figure}
\centering
\begin{tabular}{cc}
  \includegraphics[scale = 0.4]{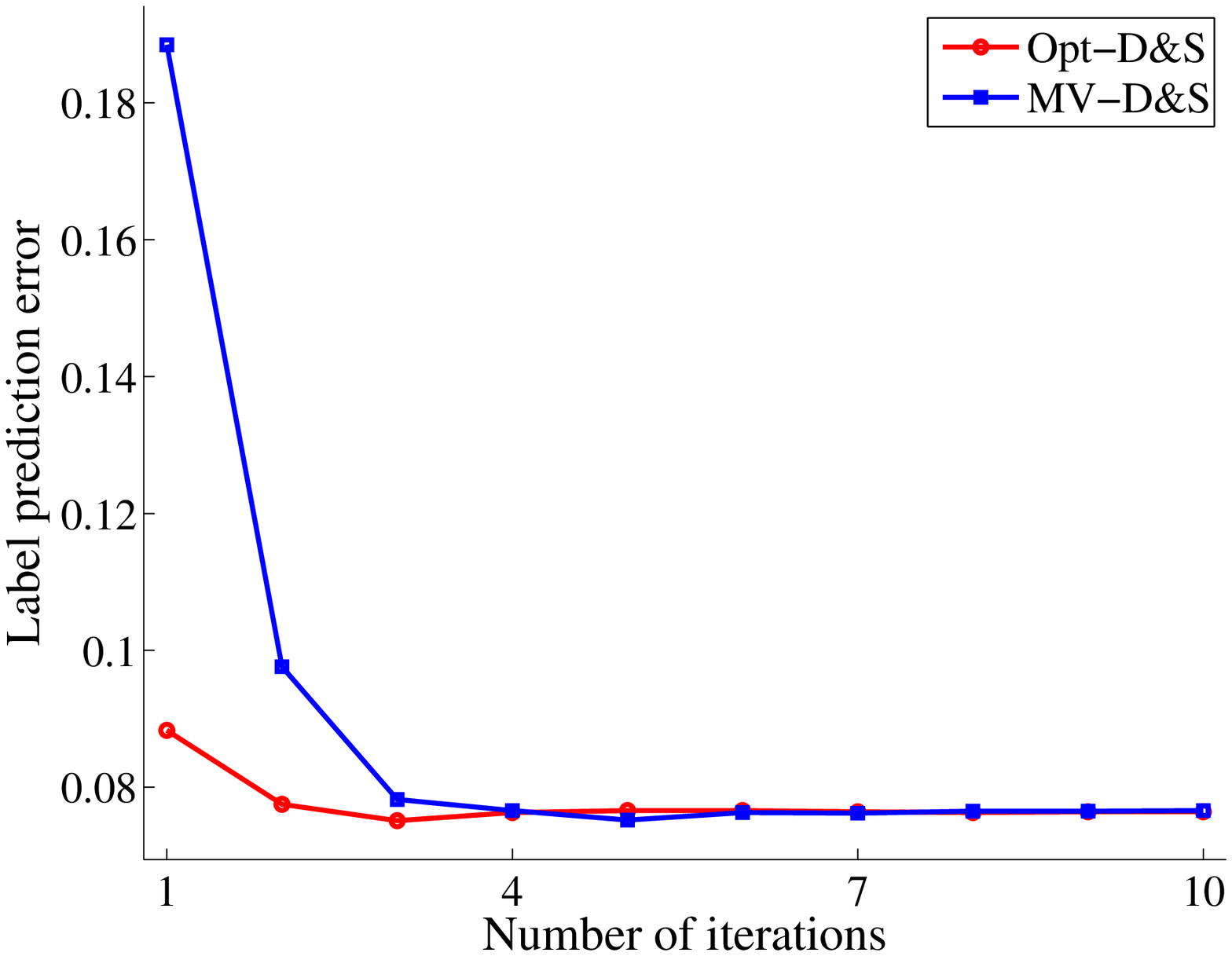} & \includegraphics[scale = 0.4]{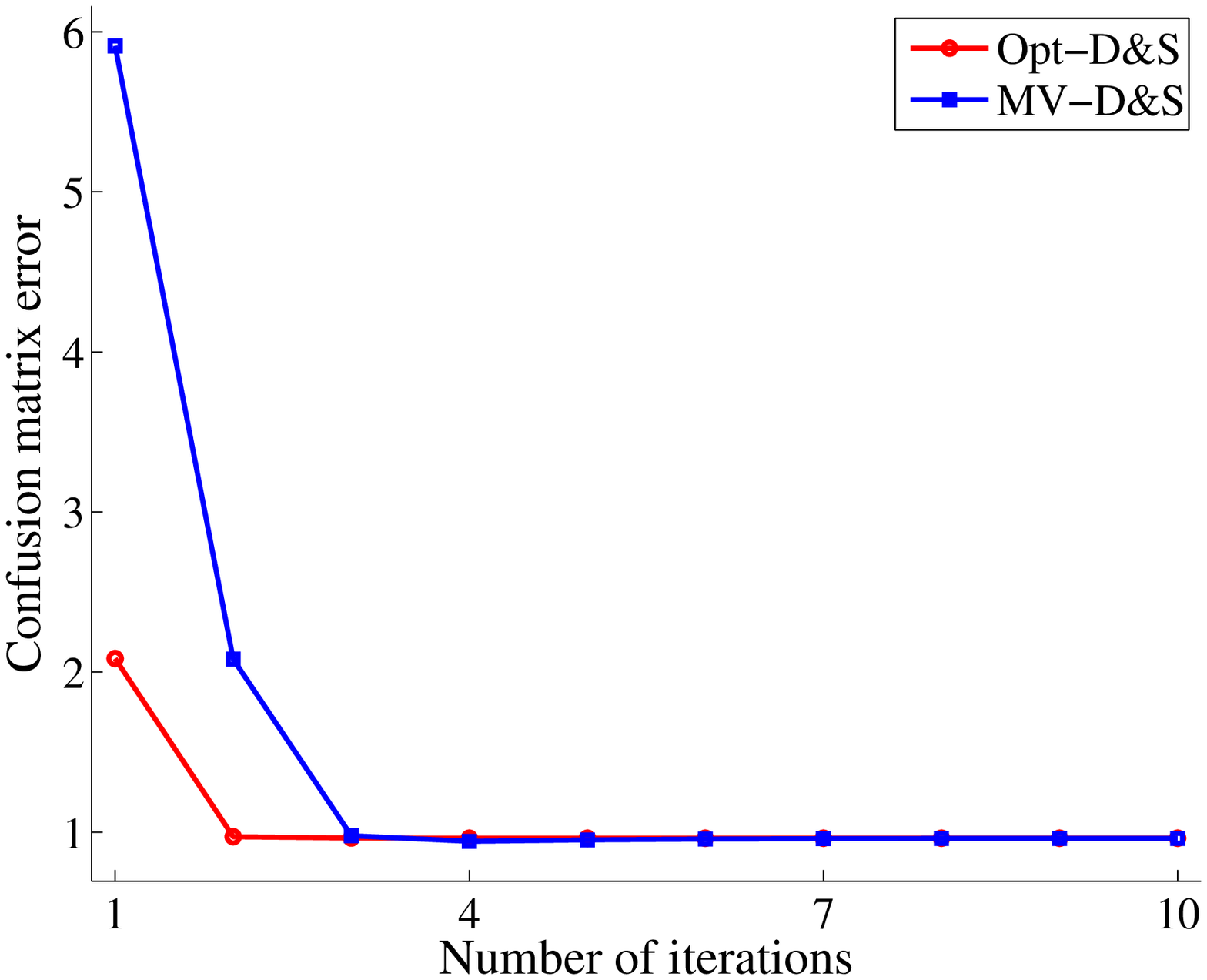}\\
  (a)&(b)
\end{tabular}
\caption{Comparing the convergence rate of the Opt-D\&S algorithm and the MV-D\&S estimator on synthetic dataset with $\pi = 0.2$: (a) convergence
of the prediction error. (b) convergence of the squared error $\sum_{i=1}^\mdim \norms{\Chat_i - C_i}_F^2$ for estimating confusion matrices.}
\label{fig:synthetic-plot}
\end{figure}

For synthetic data, we generate $\mdim = 100$ workers and $\numobs = 1000$ binary tasks. The true label of each
task is uniformly sampled from $\{1,2\}$. For each worker, the 2-by-2 confusion matrix is generated as follow:
the two diagonal entries are independently and uniformly
sampled from the interval $[0.3,0.9]$, then the non-diagonal entries
are determined to make the confusion matrix columns sum to $1$. To simulate a sparse dataset, we make each
worker label a task with probability~$\pi$. With the choice $\pi\in\{0.2,0.5,1.0\}$, we obtain three different
datasets.

We execute every algorithm independently for 10 times and average the outcomes. For the Opt-D\&S algorithm
and the MV-D\&S estimator, the estimation is outputted after 10 EM iterates.
For the group partitioning step involved in the Opt-D\&S algorithm,
the workers are randomly and evenly partitioned into three groups.

The main evaluation metric is the error of predicting the true label of items. The performance of various methods are
reported in Table~\ref{tab:synthetic-prediction-error}. On all sparsity levels, the Opt-D\&S algorithm achieves the best
accuracy, followed by the MV-D\&S estimator. All other methods are consistently worse.
It is not surprising that the Opt-D\&S algorithm and the MV-D\&S estimator yield similar accuracies,
since they optimize the same log-likelihood objective.  It is also meaningful to look at the convergence speed
of both methods, as they employ distinct initialization strategies.  Figure~\ref{fig:synthetic-plot} shows that
the Opt-D\&S algorithm converges faster than the MV-D\&S estimator, both in estimating the true labels and
in estimating confusion matrices. This is the cost that is incurred to obtain the general theoretical guarantee
associated with Opt-D\&S (recall Theorem~\ref{theorem:stage-1}).

\subsection{Real data}

\begin{table}[t]
\centering
\begin{tabular}{|c|c|c|c|c|}
  \hline
  Dataset name  & \# classes & \#  items & \#  workers & \#  worker labels \\\hline
  Bird & 2 & 108 & 39 & 4,212 \\\hline
  RTE & 2 & 800 & 164 & 8,000\\\hline
  TREC & 2 &  19,033& 762 & 88,385 \\\hline
  Dog & 4 & 807 & 52 & 7,354 \\\hline
  Web & 5 & 2,665 & 177 & 15,567 \\
  \hline
\end{tabular}
\caption{The summary of datasets used in the real data experiment.}\label{tab:real-data-summary}
\end{table}

\begin{table}[t]
\centering
\begin{tabular}{|c|c|c|c|c|c|c|}
  \hline
    & Opt-D\&S & MV-D\&S & Majority Voting & KOS & Ghosh-SVD & EigenRatio \\\hline
  Bird & {\bf 10.09} & 11.11 & 24.07 & 11.11 & 27.78 & 27.78 \\\hline
  RTE & {\bf 7.12} & 7.12 & 10.31 & 39.75 & 49.13 & 9.00\\\hline
  TREC & {\bf 29.80} & 30.02 & 34.86 & 51.96 & 42.99 & 43.96 \\\hline
  Dog & 16.89 & {\bf 16.66} & 19.58 & 31.72 & -- & --\\\hline
  Web & 15.86 & {\bf 15.74} & 26.93 & 42.93 & -- & --\\
  \hline
\end{tabular}
\caption{Error rate (\%) in predicting the true labels on real data.}\label{tab:real-prediction-error}
\end{table}

\begin{figure}[t]
\centering
\begin{tabular}{ccc}
\includegraphics[width = 0.31\textwidth]{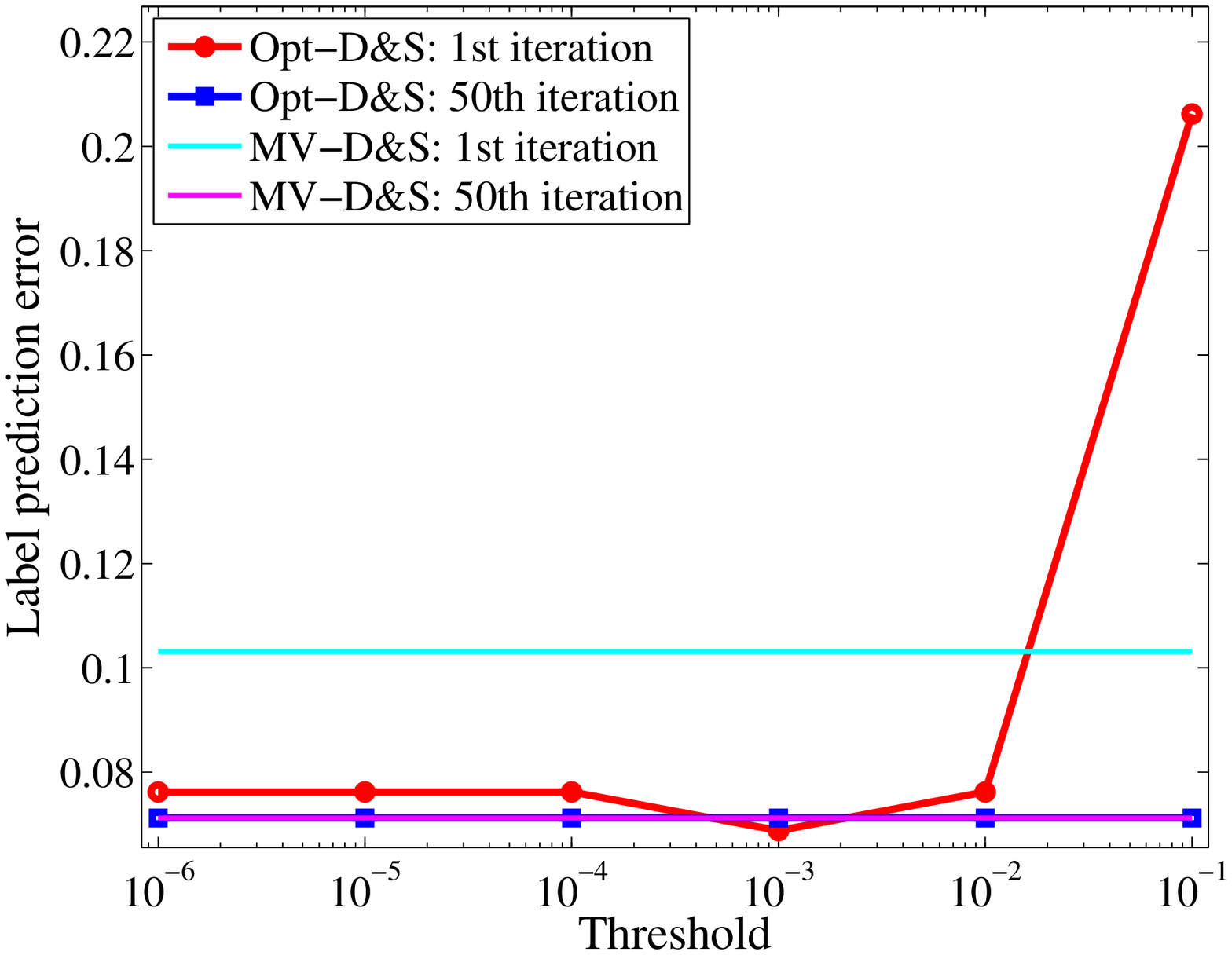}&
\includegraphics[width = 0.31\textwidth]{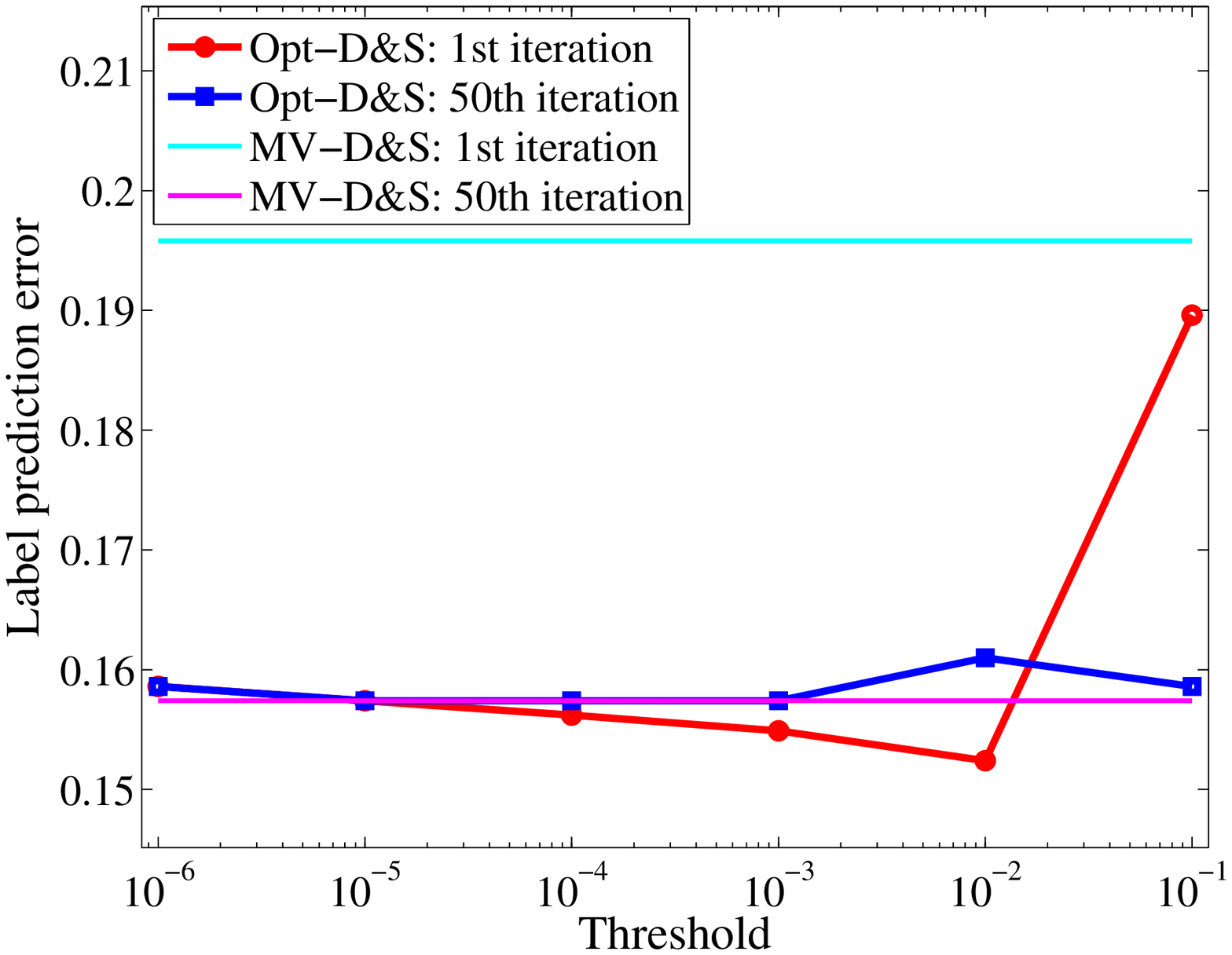}&
\includegraphics[width = 0.31\textwidth]{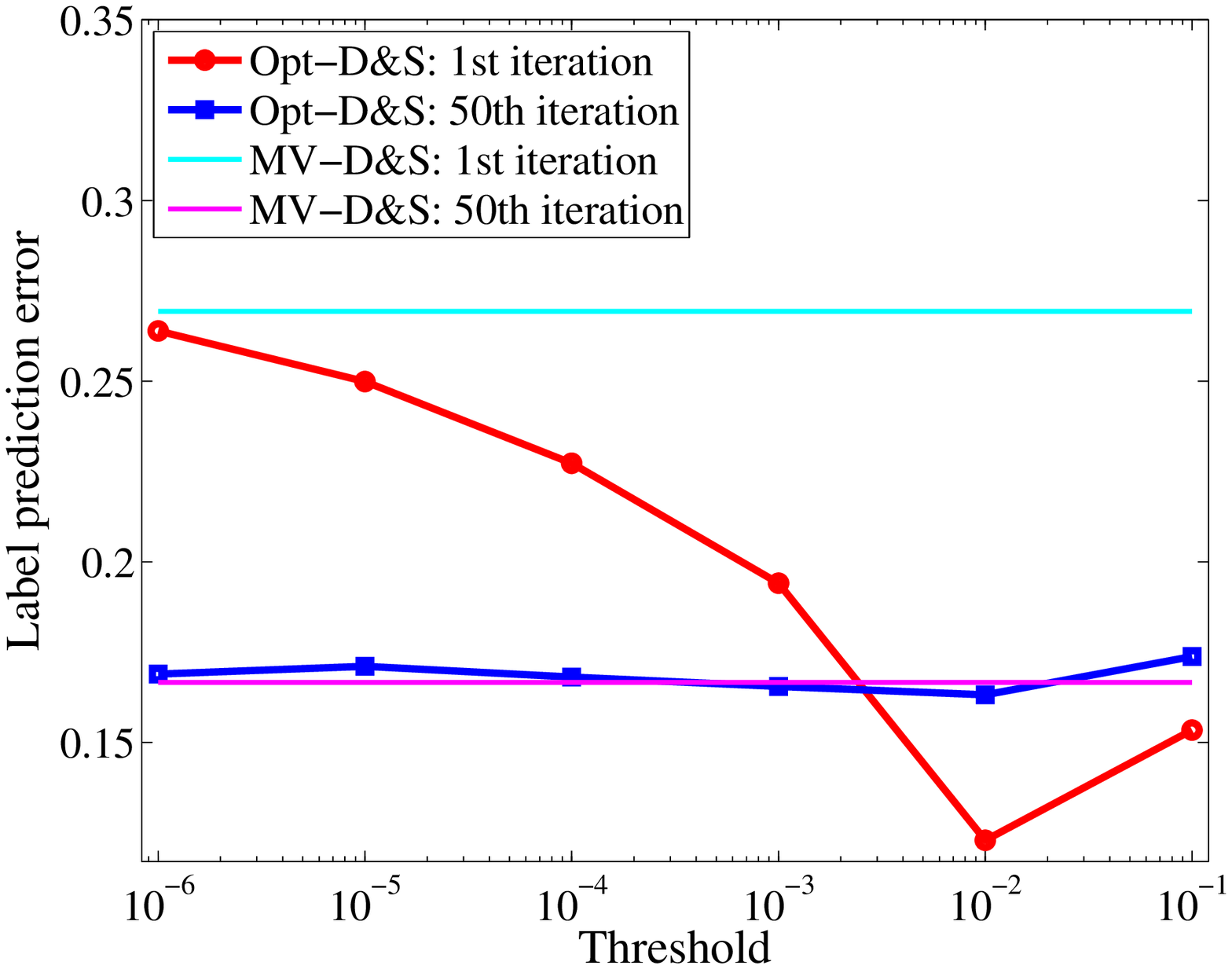}\\
(a) RTE & (b) Dog & (c) Web
\end{tabular}
\caption{Comparing the MV-D\&S estimator the Opt-D\&S algorithm with different thresholding parameter $\Delta$.
The predict error is plotted after the 1st EM update and after convergence.}\label{fig:threshold-compare-plots}
\end{figure}

For real data experiments, we compare crowdsourcing algorithms on five datasets: three binary tasks and two multi-class tasks.
Binary tasks include labeling bird species~\cite{welinder2010multidimensional} (Bird dataset), recognizing textual entailment~\cite{snow2008cheap} (RTE
dataset) and assessing the quality of documents in TREC 2011 crowdsourcing track~\cite{Lease2011Overview} (TREC dataset). Multi-class tasks include
labeling the bread of dogs from ImageNet~\cite{deng2009imagenet} (Dog dataset) and judging the relevance of web search results~\cite{zhou2012learning} (Web dataset).
The statistics for the five datasets are summarized in Table~\ref{tab:real-data-summary}.
Since the Ghost-SVD algorithm and the EigenRatio algorithm work on binary tasks, they are evaluated on the Bird, RTE and TREC dataset.
For the MV-D\&S estimator and the Opt-D\&S algorithm, we iterate their EM steps until convergence.

Since entries of the confusion matrix are positive, we find it helpful to incorporate this prior knowledge
into the initialization stage of the Opt-D\&S algorithm. In particular, when estimating the confusion matrix
entries by equation~\eqref{eqn:compute-confusion-matrix-estimate}, we add an extra checking step before the normalization,
examining if the matrix components are greater than or equal to a small threshold $\Delta$. For components that are
smaller than $\Delta$, they are reset to $\Delta$. The default choice of the thresholding parameter is
$\Delta = 10^{-6}$. Later, we will compare the Opt-D\&S algorithm with respect to different
choices of $\Delta$. It is important to note that this modification doesn't change our theoretical
result, since the thresholding step doesn't take effect if the initialization error is bounded by
Theorem~\ref{theorem:stage-1}.

Table~\ref{tab:real-prediction-error} summarizes the performance of each method.
The MV-D\&S estimator and the Opt-D\&S algorithm consistently outperform the other methods
in predicting the true label of items.  The KOS algorithm, the Ghost-SVD algorithm and the
EigenRatio algorithm yield poorer performance, presumably due to the fact that they
rely on idealized assumptions that are not met by the real data.
In Figure~\ref{fig:threshold-compare-plots}, we compare the Opt-D\&S
algorithm with respect to different thresholding parameters $\Delta\in \{10^{-i}\}_{i=1}^6$.
We plot results for three datasets (RET, Dog, Web), where the performance of the MV-D\&S
estimator is equal to or slightly better than that of Opt-D\&S.
The plot shows that the performance of the Opt-D\&S algorithm is stable after convergence.
But at the first EM iterate, the error rates are more sensitive to the choice of~$\Delta$.
A proper choice of $\Delta$ makes the Opt-D\&S algorithm perform better than MV-D\&S.
The result suggests that a proper initialization combining with one EM iterate is good enough
for the purposes of prediction. In practice, the best choice of $\Delta$
can be obtained by cross validation.

\section{Conclusions}
\label{sec:conclusion}

Under the generative model proposed by Dawid and Skene \cite{dawid1979maximum}, we propose an optimal algorithm for inferring true labels in the multi-class crowd labeling setting. Our method utilizes the method of moments to construct the initial estimator for the EM algorithm. We proved that our method achieves the optimal rate with only one iteration of the EM algorithm. 

To the best of our knowledge, this work provides the first instance of a provable convergence for a latent variable model in which EM is initialized the method of moments. One-step EM initialized by the method of moments not only leads to better estimation error in terms of the dependence on the condition number of the second-order moment matrix but it also computationally more attractive than the standard one-step estimator obtained via a Newton-Raphson step. It is interesting to explore whether a properly initialized one-step EM algorithm can achieve the optimal rate for other latent variable models such as latent Dirichlet allocation or other mixed membership models.


\newpage
\appendix
\section{Proof of Theorem~\ref{theorem:stage-1}}
\label{sec:proof-stage-1}

If $a\neq b$, it is easy to verify that $S_{ab} = \Cdiam_a W (\Cdiam_b)^T = \E[Z_{aj}\otimes Z_{bj}]$. Furthermore, we can upper bound
the spectral norm of $S_{ab}$, namely
\begin{align*}
  \lop{S_{ab}} \leq \sum_{l=1}^\kdim w_l \ltwo{\mudiam_{al}}\ltwo{\mudiam_{bl}} \leq \sum_{l=1}^\kdim w_l \norm{\mudiam_{al}}_1\norm{\mudiam_{bl}}_1 \leq 1.
\end{align*}
For the same reason, it can be shown that $\lop{T_{abc}} \leq 1$.

Our proof strategy is briefly described as follow: we upper bound the estimation error for computing empirical moments~\eqref{eqn:compute-moments-1}-\eqref{eqn:compute-moments-4} in Lemma~\ref{lemma:moment-estimation-error},
and upper bound the estimation error for tensor decomposition in Lemma~\ref{lemma:tensor-decomposition-error}. Then, we combine both lemmas to upper bound
the error of formula~\eqref{eqn:compute-confusion-matrix-estimate}.

\begin{lemma}\label{lemma:moment-estimation-error}
Given a permutation $(a,b,c)$ of $(1,2,3)$, for any scalar $\little \leq \sigmadown/2$, the second and the third moments $\Mhat_2$ and $\Mhat_3$ computed by equation~\eqref{eqn:compute-moments-3}
and~\eqref{eqn:compute-moments-4} are bounded as
\begin{align}
  \max\{\lops{\Mhat_2 - M_2 }, \lops{\Mhat_3 - M_3 }\}\leq 31 \little/\sigmadown^3
\end{align}
with probability at least $1 - \delta$, where $\delta = 6\exp(-(\sqrt{\numobs}\little - 1)^2) + \kdim\exp(-(\sqrt{\numobs/\kdim}\little - 1)^2)$.
\end{lemma}

\begin{lemma}\label{lemma:tensor-decomposition-error}
Suppose that $(a,b,c)$ is permutation of $(1,2,3)$. For any scalar $\little \leq \dgap/2$, if the empirical moments $\Mhat_2$ and $\Mhat_3$ satisfy
\begin{align}\label{eqn:moment-estimation-condition}
  &\max\{\lops{\Mhat_2 - M_2 }, \lops{\Mhat_3 - M_3 }\} \leq \little \hugeconst\\
  \mbox{for}\quad &\hugeconst \defeq \min\left\{ \frac{1}{2}, \frac{2\sigmadown^{3/2}}{15\kdim (24\sigmadown^{-1} + 2\sqrt{2})}, \frac{\sigmadown^{3/2}}{4\sqrt{3/2}\sigmadown^{1/2} + 8\kdim(24/\sigmadown + 2\sqrt{2})}\nonumber \right\}
\end{align}
then the estimates $\Cdiamhat_c$ and $\What$ are bounded as
\begin{align*}
  \lops{\Cdiamhat_c - \Cdiam_c} \leq \sqrt{\kdim}\little \qquad \mbox{and} \qquad \lops{\What - W} \leq \little.
\end{align*}
with probability at least $1-\delta$, where $\delta$ is defined in Lemma~\ref{lemma:moment-estimation-error}.
\end{lemma}

Combining Lemma~\ref{lemma:moment-estimation-error}, Lemma~\ref{lemma:tensor-decomposition-error}, if we choose a scalar $\altlittle$ satisfying
\begin{align}
    \altlittle \leq \min\{\dgap/2, \pimin \wmin \sigmadown/(36\kdim) \}, \label{eqn:stage-1-little-requirement}
\end{align}
then the estimates $\Cdiamhat_g$ (for $g=1,2,3$) and $\What$ satisfy that
\begin{align}\label{eqn:group-cfg-and-weight-estimation-error}
  \lops{\Cdiamhat_g - \Cdiam_g} \leq \sqrt{\kdim} \altlittle \qquad \mbox{and} \qquad \lops{\What - W} \leq \altlittle.
\end{align}
with probability at least $1-6\delta$, where
\[
    \delta = (6 + \kdim)\exp\Big(-(\sqrt{\numobs/\kdim}\altlittle \hugeconst \sigmadown^3 /31  - 1)^2\Big).
\]
To be more precise, we obtain the bound~\eqref{eqn:group-cfg-and-weight-estimation-error} by plugging $\little \defeq \altlittle \hugeconst \sigmadown^3 /31$
into Lemma~\ref{lemma:moment-estimation-error}, then plugging $\little \defeq \altlittle $ into Lemma~\ref{lemma:tensor-decomposition-error}.
The high probability statement is obtained by apply union bound.
\\

Assuming inequality~\eqref{eqn:group-cfg-and-weight-estimation-error}, for any $a\in\{1,2,3\}$,
since $\lops{\Cdiam_a} \leq \sqrt{\kdim},\lops{\Cdiamhat_a - \Cdiam_a}\leq \sqrt{\kdim}\altlittle$ and $\lop{W} \leq 1, \lop{\What} \leq \altlittle$,
Lemma~\ref{lemma:matrix-multiplication} (the preconditions are satisfied by inequality~\eqref{eqn:stage-1-little-requirement}) implies that
\begin{align*}
  \lop{ \What  \Cdiamhat_a - W \Cdiam_a} \leq 4\sqrt{\kdim}\altlittle,
\end{align*}
Since condition~\eqref{eqn:stage-1-little-requirement} implies
\begin{align*}
 \lops{ \What  \Cdiamhat_a - W \Cdiam_a} \leq 4 \sqrt{\kdim} \altlittle\leq \sqrt{\wmin \sigmadown}/2 \leq \sigma_\kdim(W\Cdiam_a)/2
\end{align*}
Lemma~\ref{lemma:matrix-inversion} yields that
\begin{align*}
  \lop{ \left(\What  \Cdiamhat_a \right)^{-1} - \left( W \Cdiam_a \right)^{-1}} \leq \frac{8\sqrt{\kdim}\altlittle}{\wmin \sigmadown}.
\end{align*}
By Lemma~\ref{lemma:tensor-concentration-lemma}, for any $i\in[\mdim]$, the concentration bound
\begin{align*}
  \lop{ \frac{1}{\numobs}\sum_{j=1}^\numobs z_{ij}Z_{aj}^T - \E[z_{ij}Z_{aj}^T]} \leq \altlittle
\end{align*}
holds with probability at least $1 - \mdim \exp(-(\sqrt{\numobs} \altlittle - 1)^2)$. Combining the above two inequalities
with Proposition~\ref{prop:recover-individual-matrix}, then applying Lemma~\ref{lemma:matrix-multiplication} with preconditions
\begin{align*}
  \lop{(W  \Cdiam_a )^{-1}} \leq  \frac{1}{\wmin \sigmadown} \quad \mbox{and} \quad \lop{\E\left[ z_{ij}Z_{aj}^T\right]} \leq  1,
\end{align*}
we have
\begin{align}\label{eqn:second-last-bound}
  \Big\| \underbrace{ \Big(\frac{1}{\numobs}\sum_{j=1}^\numobs z_{ij}Z_{aj}^T \Big)\Big(\What  \Cdiamhat_a \Big)^{-1}}_{\Ghat}  - \pi_i C_i \Big\|_{\rm \scriptscriptstyle op} \leq \frac{18\sqrt{\kdim} \altlittle}{\wmin \sigmadown}.
\end{align}
Let $\Ghat\in \R^{\kdim\times \kdim}$ be the first term on the left hand side of inequality~\eqref{eqn:second-last-bound}. Each column of $\Ghat$, denoted by $\Ghat_l$, is an estimate of $\pi_i \mu_{il}$. The  $\ell_2$-norm estimation error is bounded by $\frac{18\sqrt{\kdim} \altlittle}{\wmin \sigmadown}$.
Hence, we have
\begin{align}\label{eqn:normalization-factor-bound}
  \lone{\Ghat_l - \pi_i \mu_{il}} \leq \sqrt{\kdim}\ltwos{\Ghat_l - \pi_i \mu_{il}} \leq  \sqrt{\kdim} \lops{\Ghat - \pi_i C_i} \leq \frac{18\kdim \altlittle}{\wmin \sigmadown},
\end{align}
and consequently, using the fact that $\sum_{c=1}^\kdim \mu_{ilc}=1$, we have
\begin{align}
  \ltwo{{\rm normalize}(\Ghat_l) - \mu_{il}} &= \ltwo{ \frac{\Ghat_l}{\pi_i + \sum_{c=1}^\kdim \left(\Ghat_{lc} - \pi_i \mu_{ilc} \right)} - \mu_{il} } \nonumber \\
  &\leq \frac{\ltwos{\Ghat_l - \pi_i\mu_{il} } + \lone{\Ghat_l - \pi_i \mu_{il}} \ltwos{\mu_{il}}}{\pi_i - \lone{\Ghat_l - \pi_i \mu_{il}}} \nonumber \\
  &\leq \frac{72\kdim \altlittle}{\pimin\wmin \sigmadown} \label{eqn:last-ell-2-bound}
\end{align}
where the last step combines inequalities~\eqref{eqn:second-last-bound},~\eqref{eqn:normalization-factor-bound} with the bound
$\frac{18\kdim \altlittle}{\wmin \sigmadown} \leq \pi_i/2$ from condition~\eqref{eqn:stage-1-little-requirement}, and uses the fact that $\ltwos{\mu_{il}}\leq 1$. \\

\noindent Note that inequality~\eqref{eqn:last-ell-2-bound} holds with probability at least
\begin{align*}
  1 - (36 + 6\kdim)\exp\Big(-(\sqrt{\numobs/\kdim}\altlittle \hugeconst \sigmadown^3 /31  - 1)^2\Big) - \mdim \exp(-(\sqrt{\numobs} \altlittle - 1)^2).
\end{align*}
It can be verified that $H\geq \frac{\sigmadown^{5/2}}{230\kdim}$. Thus, the above
expression is lower bounded by
\begin{align*}
  1 - (36 + 6\kdim + \mdim)\exp\Big(-\Big( \frac{\sqrt{\numobs} \altlittle \sigmadown^{11/2}}{31\times 230\cdot \kdim^{3/2}}  - 1\Big)^2\Big),
\end{align*}
If we represent this probability in the form of $1-\delta$, then
\begin{align}\label{eqn:error-bounded-by-sample-complexity}
  \altlittle = \frac{31\times 230\cdot \kdim^{3/2}}{\sqrt{\numobs} \sigmadown^{11/2}} \left( 1 + \sqrt{\log((36+6\kdim+\mdim)/\delta)}\right).
\end{align}
Combining condition~\eqref{eqn:stage-1-little-requirement} and inequality~\eqref{eqn:last-ell-2-bound},
we find that to make $\norms{\Chat - C}_\infty$ bounded by $\little$, it is sufficient to choose
$\little_1$ such that
\begin{align*}
  \altlittle \leq \min\left\{ \frac{\little\pimin\wmin \sigmadown}{72\kdim}, \frac{\dgap}{2}, \frac{\pimin\wmin \sigmadown}{36\kdim} \right\}
\end{align*}
This condition can be further simplified to
\begin{align}\label{eqn:condition-on-altlittle}
  \altlittle\leq \frac{\little\pimin\wmin \sigmadown}{72\kdim}
\end{align}
for small $\little$, that is $\little\leq \min\left\{ \frac{36\dgap\kdim}{\pimin\wmin \sigmadown},2 \right\}$.
According to equation~\eqref{eqn:error-bounded-by-sample-complexity}, the condition~\eqref{eqn:condition-on-altlittle} will be satisfied
if
\begin{align*}
  \sqrt{\numobs} \geq \frac{72\times 31\times 230\cdot \kdim^{5/2}}{\little \pimin \wmin \sigmadown^{13/2}} \left( 1 + \sqrt{\log((36+6\kdim+\mdim)/\delta)}\right).
\end{align*}
Taking square over both sides of the inequality completes the proof.

\subsection{Proof of Lemma~\ref{lemma:moment-estimation-error}}

Throughout the proof, we assume that the following concentration bound holds: for any distinct indices $(a',b')\in \{1,2,3\}$, we have
\begin{align}
\lop{\frac{1}{n}\sum_{j=1}^\numobs Z_{a'j}\otimes Z_{b'j} - \E[Z_{a'j}\otimes Z_{b'j}]} &\leq \little \label{eqn:matrix-concentration-condition}
\end{align}
By Lemma~\ref{lemma:tensor-concentration-lemma} and the union bound, this event happens with probability at
least $1 - 6\exp(-(\sqrt{\numobs}\little - 1)^2)$.
By the assumption that $\little \leq \sigmadown/2 \leq \sigma_k(S_{ab})/2$ and Lemma~\ref{lemma:matrix-inversion}, we have
\begin{align*}
  \lop{\frac{1}{n}\sum_{j=1}^\numobs Z_{cj}\otimes Z_{bj} - \E[Z_{cj}\otimes Z_{bj}]} &\leq \little \quad\mbox{and}\\
  \lop{\left(\frac{1}{n}\sum_{j=1}^\numobs Z_{aj}\otimes Z_{bj}\right)^{-1} - \left(\E[Z_{aj}\otimes Z_{bj}]\right)^{-1}} &\leq \frac{2\little}{\sigma^2_\kdim(\Smat_{ab})}
\end{align*}
Under the preconditions
\begin{align*}
  \lop{\E[Z_{cj}\otimes Z_{bj}]} \leq 1 \quad \mbox{and} \quad \lop{(\E[Z_{aj}\otimes Z_{bj}])^{-1}} \leq \frac{1}{\sigma_\kdim(\Smat_{ab})},
\end{align*}
Lemma~\ref{lemma:matrix-multiplication} implies that
\begin{align}
  &\lop{\left(\frac{1}{n}\sum_{j=1}^\numobs Z_{cj}\otimes Z_{bj}\right)\left(\frac{1}{n} Z_{aj}\otimes Z_{bj}\right)^{-1} - \E[Z_{cj}\otimes Z_{bj}](\E[Z_{aj}\otimes Z_{bj}])^{-1} } \nonumber\\
  &\qquad \leq 2 \left(\frac{\little}{\sigma_k(S_{ab})} + \frac{2\little }{\sigma^2_\kdim(\Smat_{ab})}\right)\leq 6\little/\sigmadown^2 \label{eqn:zprime-concentration-1}
\end{align}
and for the same reason, we have
\begin{align}
  &\lop{\left(\frac{1}{n}\sum_{j=1}^\numobs Z_{cj}\otimes Z_{aj}\right)\left(\frac{1}{n} Z_{bj}\otimes Z_{aj}\right)^{-1} - \E[Z_{cj}\otimes Z_{aj}](\E[Z_{bj}\otimes Z_{aj}])^{-1} } \leq 6\little/\sigmadown^2 \label{eqn:zprime-concentration-2}
\end{align}
Now, let matrices $F_2$ and $F_3$ be defined as
\begin{align*}
  F_2 &\defeq \E[Z_{cj}\otimes Z_{bj}](\E[Z_{aj}\otimes Z_{bj}])^{-1},\\
  F_3 &\defeq \E[Z_{cj}\otimes Z_{aj}](\E[Z_{bj}\otimes Z_{aj}])^{-1},
\end{align*}
and let the matrix on the left hand side of inequalities~\eqref{eqn:zprime-concentration-1} and~\eqref{eqn:zprime-concentration-2} be denoted by
$\Delta_2$ and $\Delta_3$, we have
\begin{align*}
  &\lop{ \widehat{Z}'_{aj}\otimes \widehat{Z}'_{bj} - F_2\left( Z_{aj}\otimes Z_{bj}\right)F_3^T} = \lop{ \Big( F_2 + \Delta_2 \Big) \left( Z_{aj}\otimes Z_{bj} \right)\Big( F_3 + \Delta_3 \Big)^T - F_2\left( Z_{aj}\otimes Z_{bj}\right)F_3^T}\\
  &\qquad \leq \lop{Z_{aj}\otimes Z_{bj} } \left( \lop{ \Delta_2 }\lop{F_3 + \Delta_2} + \lop{F_2}\lop{ \Delta_3 }\right) \leq 30\little \lop{Z_{aj}\otimes Z_{bj}} /\sigmadown^3.
\end{align*}
where the last steps uses inequality~\eqref{eqn:zprime-concentration-1},~\eqref{eqn:zprime-concentration-2} and the fact that
$\max\{\lop{F_2},\lop{F_3}\} \leq 1/\sigmadown$ and
\begin{align*}
  \lop{F_3 + \Delta_2} \leq \lops{F_3} + \lops{\Delta_2}\leq 1/\sigmadown + 6\little/\sigmadown^2 \leq 4/\sigmadown.
\end{align*}
To upper bound the norm $\lop{Z_{aj}\otimes Z_{bj}}$, notice that
\begin{align*}
  \lop{Z_{aj}\otimes Z_{bj}} \leq \ltwo{Z_{aj}} \ltwo{Z_{bj}} \leq \norm{Z_{aj}}_1\norm{Z_{bj}}_1 \leq 1.
\end{align*}
Consequently, we have
\begin{align}\label{eqn:zprime-concentration-3}
  \lop{ \widehat{Z}'_{aj}\otimes \widehat{Z}'_{bj} - F_2\left( Z_{aj}\otimes Z_{bj}\right)F_3^T} \leq 30\little /\sigmadown^3.
\end{align}
For the rest of the proof, we use inequality~\eqref{eqn:zprime-concentration-3} to bound $\Mhat_2$ and $\Mhat_3$. For the second
moment, we have
\begin{align*}
  \lop{\Mhat_2 - M_2} &\leq \frac{1}{\numobs}\sum_{j=1}^\numobs \lop{ \widehat{Z}'_{aj}\otimes \widehat{Z}'_{bj} - F_2\left( Z_{aj}\otimes Z_{bj}\right)F_3^T} + \lop{F_2\left( \frac{1}{\numobs}\sum_{j=1}^\numobs Z_{aj}\otimes Z_{bj} \right)F_3^T - M_2}\\
  &\leq  30\little /\sigmadown^3 + \lop{F_2\left( \frac{1}{\numobs} \sum_{j=1}^\numobs Z_{aj}\otimes Z_{bj} - \E[Z_{aj}\otimes Z_{bj}] \right)F_3^T}\\
  &\leq  30\little /\sigmadown^3 + \little/\sigmadown^2 \leq 31\little /\sigmadown^3.
\end{align*}
For the third moment, we have
\begin{align}\label{eqn:third-moment-error-expansion}
  \Mhat_3 - M_3 &= \frac{1}{\numobs}\sum_{j=1}^\numobs \left( \widehat{Z}'_{aj}\otimes \widehat{Z}'_{bj} - F_2\left( Z_{aj}\otimes Z_{bj}\right)F_3^T \right)\otimes Z_{cj} \nonumber\\
  &\qquad + \left( \frac{1}{\numobs}\sum_{j=1}^\numobs
  F_2\left( Z_{aj}\otimes Z_{bj}\right)F_3^T\otimes Z_{cj} - \E\left[ F_2\left( Z_{aj}\otimes Z_{bj}\right)F_3^T\otimes Z_{cj} \right]\right).
\end{align}
We examine the right hand side of equation~\eqref{eqn:third-moment-error-expansion}. The first term is bounded as
\begin{align}\label{eqn:zprime-concentration-4}
  \lop{\left( \widehat{Z}'_{aj}\otimes \widehat{Z}'_{bj} - F_2\left( Z_{aj}\otimes Z_{bj}\right)F_3^T \right)\otimes Z_{cj}} &\leq \lop{\widehat{Z}'_{aj}\otimes \widehat{Z}'_{bj} - F_2\left( Z_{aj}\otimes Z_{bj}\right)F_3^T}\ltwo{Z_{cj}} \nonumber\\
  &\leq 30\little /\sigmadown^3.
\end{align}
For the second term, since $\ltwos{F_2Z_{aj}}\leq 1/\sigmadown$, $\ltwos{F_3Z_{bj}}\leq 1/\sigmadown$ and $\ltwos{Z_{cj}} \leq 1$, Lemma~\ref{lemma:tensor-concentration-lemma} implies that
\begin{align}\label{eqn:tensor-concentration-result}
  \lop{ \frac{1}{\numobs}\sum_{j=1}^\numobs F_2\left( Z_{aj}\otimes Z_{bj}\right)F_3^T\otimes Z_{cj} - \E\left[ F_2\left( Z_{aj}\otimes Z_{bj}\right)F_3^T\otimes Z_{cj} \right]} \leq \little/\sigmadown^2
\end{align}
with probability at least $1 - \kdim \exp( - (\sqrt{\numobs/\kdim}\little - 1)^2)$. Combining inequalities~\eqref{eqn:zprime-concentration-4}
and~\eqref{eqn:tensor-concentration-result}, we have
\begin{align*}
  \lop{\Mhat_3 - M_3} \leq 30\little /\sigmadown^3 + \little/\sigmadown^2 \leq 31\little /\sigmadown^3.
\end{align*}
Applying union bound to all high-probability events completes the proof.

\subsection{Proof of Lemma~\ref{lemma:tensor-decomposition-error}}

Chaganty and Liang~(Lemma 4 in \cite{chaganty2013spectral}) have proved that when condition~\eqref{eqn:moment-estimation-condition} holds,
the tensor decomposition method of Algorithm~\ref{alg:estimating-confusion-matrix} outputs $\{\mudiamhat_h,\weighthat_h\}_{h=1}^\kdim$, such that with probability at least $1-\delta$,
a permutation $\pi$ satisfies
\begin{align*}
  \ltwos{\mudiamhat_{h} - \mudiam_{c\pi(h)}} \leq \little \qquad \mbox{and} \qquad \norm{\weighthat_h - w_{\pi(h)} }_\infty \leq \little.
\end{align*}
Note that the constant $\hugeconst$ in Lemma~\ref{lemma:tensor-decomposition-error} is obtained by plugging upper bounds
$\lops{M_2}\leq 1$ and $\lops{M_3}\leq 1$ into Lemma 4 of Chaganty and Liang~\cite{chaganty2013spectral}.

The $\pi(h)$-th component of $\mudiam_{c\pi(h)}$ is greater than other components of $\mudiam_{c\pi(h)}$, by a margin of $\dgap$.
Assuming
$\little \leq \dgap/2$, the greatest component of $\mudiamhat_{h}$ is its $\pi(h)$-th component. Thus, Algorithm~\ref{alg:estimating-confusion-matrix} is able to
correctly
estimate the $\pi(h)$-th column of $\Cdiamhat_c$ by the vector $\mudiamhat_{h}$. Consequently, for every column of $\Cdiamhat_c$, the $\ell_2$-norm error is bounded by
$\little$. Thus, the spectral-norm error of $\Cdiamhat_c$ is bounded by $\sqrt{\kdim}\little$. Since $W$ is a diagonal matrix and $\norm{\weighthat_h - w_{\pi(h)} }_\infty \leq \little$,
we have $\lops{\What - W} \leq \little$.

\section{Proof of Theorem~\ref{theorem:stage-2}}
\label{sec:proof-stage-2}

We define two random events that will be shown holding with high probability:
\begin{align*}
  &\event_1:~\sum_{i=1}^\mdim \sum_{c=1}^\kdim \indicator(z_{ij} = e_c)\log(\mu_{iy_jc}/\mu_{ilc}) \geq \mdim \Dbar/2 \qquad \mbox{for all $j\in[\numobs]$ and $l\in [\kdim]\backslash\{y_j\}$}.\\
  &\event_2:~\Big| \sum_{j=1}^\numobs \indicator(y_j = l)\indicator(z_{ij} = e_c) - \numobs w_l\pi_i\mu_{ilc} \Big| \leq \numobs \tfac \qquad \mbox{for all $(i,l,c)\in [\mdim]\times[\kdim]^2$}.
\end{align*}
where $\tfac > 0$ are scalars to be specified later. We define $\tmin$ to be the smallest
element among $\{\tfac\}$.
Assuming that $\event_1\cap\event_2$ holds, the following lemma shows
that performing updates~\eqref{eqn:q_update} and~\eqref{eqn:p_update} attains the desired level of accuracy.
See Section~\ref{sec:proof-update-stability} for the proof.

\begin{lemma}\label{lemma:update-stability}
  Assume that $\event_1\cap \event_2$ holds. Also assume that $\mu_{ilc} \geq \gap$ for all $(i,l,c)\in [\mdim]\times[\kdim]^2$. If $\Chat$ is initialized such that
  inequality~\eqref{eqn:initialization-accuracy-condition} holds, and scalars $\tfac$ satisfy
  \begin{align}\label{eqn:concentration-accuracy-condition}
    2 \exp\Big( -\mdim\Dbar/4 + \log(\mdim) \Big) \leq \tfac\leq \mass \min\left\{ \frac{\gap}{8},\frac{\gap \Dbar}{64}\right\}
  \end{align}
  Then by alternating updates~\eqref{eqn:q_update} and~\eqref{eqn:p_update} for at least one round, the estimates $\Chat$ and $\qhat$ are bounded as
  \begin{align*}
    |\muhat_{il} - \mu_{ilc}| &\leq 4\tfac/(\pi_i w_l). \qquad &&\mbox{for all $i\in[\mdim]$, $l\in[\kdim]$, $c\in [\kdim]$}.\\
    \max_{l\in[\kdim]}\{ |\qhat_{jl} - \indicator(y_j = l)|\}  &\leq \exp\left( -\mdim\Dbar/4 + \log(\mdim) \right) \qquad &&\mbox{for all $j\in [\numobs]$}.
  \end{align*}
\end{lemma}

Next, we characterize the probability that events $\event_1$ and $\event_2$ hold.
For measuring $\mprob[\event_1]$,
we define auxiliary variable $s_i \defeq \sum_{c=1}^\kdim \indicator(z_{ij} = e_c)\log(\mu_{iy_jc}/\mu_{ilc})$. It is
straightforward to see that $s_1,s_2,\dots,s_\mdim$ are mutually independent on any value of $y_j$, and each $s_i$ belongs to the interval
$[0,\log(1/\gap)]$. it is easy to verify that
\begin{align*}
  \E\left[ \sum_{i=1}^\mdim s_i \Big| y_i \right] = \sum_{i=1}^\mdim \pi_i \KL{\mu_{iy_j},\mu_{il}}.
\end{align*}
We denote the right hand side of the above equation by $D$.
The following lemma shows that the second moment of $s_i$ is bounded by the KL-divergence between labels.

\begin{lemma}\label{lemma:bound-variance}
Conditioning on any value of $y_j$, we have
\[
    \E[s_i^2 | y_i] \leq \frac{2\log(1/\gap)}{1-\gap} \pi_i\KL{\mu_{iy_j}, \mu_{il}}.
\]
\end{lemma}

According to Lemma~\ref{lemma:bound-variance}, the aggregated second moment of $s_i$ is bounded by
\[
     \E\left[ \sum_{i=1}^\mdim s_i^2 \Big| y_i \right] \leq \frac{2\log(1/\gap)}{1-\gap} \sum_{i=1}^\mdim \pi_i \KL{\mu_{iy_jc}, \mu_{ilc}} = \frac{2\log(1/\gap)}{1-\gap}D
\]
Thus, applying the Bernstein inequality, we have
\begin{align*}
  \mprob\Big[ \sum_{i=1} s_i \geq D/2 | y_i\Big] \geq 1 - \exp\left(-\frac{\frac{1}{2}(D/2)^2}{\frac{2\log(1/\gap)}{1-\gap}D+ \frac{1}{3}(2\log(1/\gap))(D/2)} \right),
\end{align*}
Since $\gap \leq 1/2$ and $D \geq \mdim \Dbar$, combining the above inequality with the union bound, we have
\begin{align}\label{eqn:event1-prob-bound}
  \mprob[\event_1] \geq 1 - \kdim \numobs \exp\left(-\frac{\mdim \Dbar}{33 \log(1/\gap)} \right).
\end{align}
For measuring $\mprob[\event_2]$, we observe that $\sum_{j=1}^\numobs \indicator(y_j = l)\indicator(z_{ij} = e_c)$ is the sum
of $\numobs$ i.i.d.~Bernoulli random variables with mean $p\defeq \pi_i w_l \mu_{ilc}$.
Since $\tfac\leq \mass \gap/8 \leq p$, applying the Chernoff bound implies
\begin{align*}
  \mprob\left[ \Big| \sum_{j=1}^\numobs \indicator(y_j = l)\indicator(z_{ij} = e_c) - \numobs p \Big| \geq \numobs \tfac  \right]
  &\leq 2\exp( -\numobs \tfac^2 /(3p)) = 2\exp\left( - \frac{\numobs \tfac^2}{3\pi_i w_l \mu_{ilc}}\right),
\end{align*}
Summarizing the probability bounds on $\event_1$ and $\event_2$, we conclude that $\event_1\cap \event_2$ holds
with probability at least
\begin{align}\label{eqn:good-event-probability-bound}
  1 - \kdim \numobs \exp\left(-\frac{\mdim \Dbar}{33 \log(1/\gap)} \right) - \sum_{i=1}^\mdim\sum_{l=1}^\kdim 2\exp\left( - \frac{\numobs \tfac^2}{3\pi_i w_l \mu_{ilc}}\right).
\end{align}

\paragraph{Proof of Part (a)} According to Lemma~\ref{lemma:update-stability}, for $\yhat_j = y_j$ being true, it sufficient to have $\exp(-\mdim\Dbar/4 + \log(m)) < 1/2$, or equivalently
\begin{align}\label{eqn:classification-mn-requirement-1}
    \mdim > 4\log(2m)/{\Dbar}.
\end{align}
To ensure that this bound holds with probability at least $1-\delta$, expression~\eqref{eqn:good-event-probability-bound}
needs to be lower bounded by $\delta$. It is achieved if we have
\begin{align}\label{eqn:classification-mn-requirement-2}
  \mdim \geq \frac{ 33 \log(1/\gap)\log(2\kdim\numobs/\delta) }{\Dbar} \quad \mbox{and} \quad \numobs \geq \frac{3\pi_i w_l \mu_{ilc} \log(2\mdim\kdim/\delta)}{\tfac^2}
\end{align}
If we choose
\begin{align}\label{eqn:regression-assign-til}
    \tfac \defeq \sqrt{ \frac{3\pi_i w_l \mu_{ilc} \log(2\mdim\kdim/\delta)}{\numobs} }.
\end{align}
then the second part of condition~\eqref{eqn:classification-mn-requirement-2} is guaranteed.
To ensure that $\tfac$ satisfies condition~\eqref{eqn:concentration-accuracy-condition}. We need to have
\begin{align*}
  \sqrt{ \frac{3\pi_i w_l \mu_{ilc} \log(2\mdim\kdim/\delta)}{\numobs} } &\geq 2\exp\Big( -\mdim\Dbar/4 + \log(\mdim) \Big) \quad\mbox{and}\\
  \sqrt{ \frac{3\pi_i w_l \mu_{ilc} \log(2\mdim\kdim/\delta)}{\numobs} } &\leq \mass \minac/4.
\end{align*}
The above two conditions  requires that $\mdim$ and $\numobs$ satisfy
\begin{align}
  \mdim &\geq \frac{4\log(\mdim\sqrt{2\numobs/(3\pimin\wmin\log(2\mdim\kdim/\delta))})}{\Dbar} \label{eqn:regression-mn-requirement-1}\\
  \numobs &\geq \frac{48\log(2\mdim\kdim/\delta)}{\pimin\wmin\minac^2} \label{eqn:regression-mn-requirement-2}
\end{align}
The four conditions~\eqref{eqn:classification-mn-requirement-1},~\eqref{eqn:classification-mn-requirement-2},~\eqref{eqn:regression-mn-requirement-1}
and~\eqref{eqn:regression-mn-requirement-2} are simultaneously satisfied if we have
\begin{align*}
  \mdim &\geq \frac{\max\{ 33 \log(1/\gap)\log(2\kdim\numobs/\delta), 4\log(2\mdim \numobs)  \}}{\Dbar}\quad\mbox{and}\\
  \numobs &\geq \frac{48\log(2\mdim\kdim/\delta)}{\pimin\wmin\minac^2}.
\end{align*}
Under this setup, $\yhat_j = y_j$ holds for all $j\in [\numobs]$ with probability at least $1-\delta$.

\paragraph{Proof of Part (b)} If
$\tfac$ is set by equation~\eqref{eqn:regression-assign-til}, combining Lemma~\ref{lemma:update-stability} with this assignment, we have
\begin{align*}
  (\muhat_{ilc} - \mu_{ilc})^2 \leq \frac{48\mu_{ilc} \log(2\mdim\kdim/\delta)}{\pi_i w_l \numobs}
\end{align*}
with probability at least $1-\delta$. Summing both sides of the inequality over $c=1,2,\dots,\kdim$ completes the
proof.

\subsection{Proof of Lemma~\ref{lemma:update-stability}}\label{sec:proof-update-stability}

To prove Lemma~\ref{lemma:update-stability}, we look into the consequences of update~\eqref{eqn:q_update} and update~\eqref{eqn:p_update}.
We prove two important lemmas, which show that both updates provide good estimates if they are properly initialized.

\begin{lemma}\label{lemma:q-concentration}
  Assume that event $\event_1$ holds. If $\mu$ and its estimate $\muhat$ satisfies
  \begin{align}\label{eqn:assu-mu-concentration}
    \mu_{ilc} \geq \gap \quad\mbox{and}\quad |\muhat_{ilc} - \mu_{ilc} | \leq \delta_1 \qquad \mbox{for all $i\in[\mdim]$, $l\in[\kdim]$, $c\in [\kdim]$},
  \end{align}
  and $\qhat$ is updated by formula~\eqref{eqn:q_update}, then $\qhat$ is bounded as:
  \begin{align}\label{eqn:conclude-q-concentration}
    \max_{l\in[\kdim]}\{ |\qhat_{jl} - \indicator(y_j = l)|\} & \leq \exp\left( -\mdim \left(\frac{\Dbar}{2} - \frac{2 \delta_1}{\gap - \delta_1}\right) + \log(\mdim) \right) \qquad \mbox{for all $j\in [\numobs]$}.
  \end{align}
\end{lemma}

\begin{proof}

For an arbitrary index $l\neq y_j$, we consider the quantity
\begin{align*}
  A_l \defeq \sum_{i=1}^\mdim \sum_{c=1}^\kdim \indicator(z_{ij} = e_c)\log(\muhat_{iy_jc}/\muhat_{ilc})
\end{align*}
By the assumption that $\event_1$ and inequality~\eqref{eqn:assu-mu-concentration} holds, we obtain that
\begin{align}
  A_l &=  \sum_{i=1}^\mdim \sum_{c=1}^\kdim \indicator(z_{ij} = e_c)\log(\mu_{iy_jc}/\mu_{ilc}) + \sum_{i=1}^\mdim \sum_{c=1}^\kdim
  \indicator(z_{ij} = e_c)\left[\log\Big(\frac{\muhat_{iy_jc}}{\mu_{iy_jc}}\Big) - \log\Big(\frac{\muhat_{ilc}}{\mu_{ilc}}\Big)\right]\nonumber\\
  &\geq \left(\sum_{i=1}^\mdim \frac{\pi_i \KL{\mu_{iy_j},\mu_{il}}}{2}\right) - 2\mdim \log\Big( \frac{\gap}{\gap - \delta_1} \Big) \geq \mdim \left(\frac{\Dbar}{2} - \frac{2 \delta_1}{\gap - \delta_1}\right).\label{eqn:lower-bound-likelihood-ratio}
\end{align}
Thus, for every index $l\neq y_j$, combining formula~\eqref{eqn:q_update} and inequality~\eqref{eqn:lower-bound-likelihood-ratio} implies that
\begin{align*}
  \qhat_{jl} \leq \frac{1}{\exp(A_l)} \leq \exp\left( -\mdim \left(\frac{\Dbar}{2} - \frac{2 \delta_1}{\gap - \delta_1}\right)\right).
\end{align*}
Consequently, we have
\begin{align*}
  \qhat_{jy_j} \geq 1 - \sum_{l\neq y_j} \qhat_{jl} \geq 1 - \exp\left( -\mdim \left(\frac{\Dbar}{2} - \frac{2 \delta_1}{\gap - \delta_1}\right) + \log(\mdim) \right).
\end{align*}
Combining the above two inequalities completes the proof.
\end{proof}

\begin{lemma}\label{lemma:mu-concentration}
  Assume that event $\event_2$ holds. If $\qhat$ satisfies
  \begin{align}\label{eqn:assu-q-concentration}
    \max_{l\in [\kdim]}\{|\qhat_{jl} - \indicator(y_j=l)|\} \leq \delta_2 \qquad \mbox{for all $j\in [\numobs]$},
  \end{align}
  and $\muhat$ is updated by formula~\eqref{eqn:p_update}, then $\muhat$ is bounded as:
  \begin{align}\label{eqn:conclude-mu-concentration}
    |\muhat_{ilc} - \mu_{ilc}| \leq \frac{2 n \tfac  +  2\numobs \delta_2}{(7/8)n \pi_i w_l   -  \numobs \delta_2}. \qquad \mbox{for all $i\in[\mdim]$, $l\in[\kdim]$, $c\in [\kdim]$}.
  \end{align}
\end{lemma}

\begin{proof}
  By formula~\eqref{eqn:p_update}, we can write $\muhat_{il} = A/B$, where
  \begin{align*}
    A \defeq \sum_{j=1}^n \qhat_{jl} \indicator(z_{ij}=e_c) \quad \mbox{and} \quad B \defeq \sum_{c'=1}^\kdim \sum_{j=1}^n \qhat_{jl} \indicator(z_{ij}=e_{c'}).
  \end{align*}
  Combining this definition with inequality~\eqref{eqn:assu-q-concentration}, we find that
  \small
  \begin{align*}
    | A - n \pi_i w_l \mu_{ilc}  \mu_{ilc} | &\leq  \Big| \sum_{j=1}^n \indicator(q_{jl} = y_j) \indicator(z_{ij}=e_c) - n \pi_i w_l \mu_{ilc}  \mu_{ilc} \Big| + \Big| \sum_{j=1}^n \qhat_{jl} \indicator(z_{ij}=e_c) - \sum_{j=1}^n \indicator(q_{jl} = y_j) \indicator(z_{ij}=e_c)\Big|\\
    &\leq  \numobs \tfac  +  \numobs \delta_2.
  \end{align*}
  \normalsize
  By the same argument, we have
  \begin{align*}
    | B - n \pi_i w_l \mu_{ilc}  | \leq  \left( \sum_{c=1}^\kdim \numobs\tfac\right)  + \numobs \delta_2.
  \end{align*}
  Combining the bound for $A$ and $B$, we obtain that
  \begin{align*}
    |\muhat_{il} - \mu_{ilc}| &= \left| \frac{\numobs \pi_i w_l \mu_{ilc} + (A - \numobs \pi_i w_l \mu_{ilc} )}{\numobs \pi_i w_l  + (B - \numobs \pi_i w_l )} - \mu_{ilc} \right| = \left| \frac{(A - \numobs \pi_i w_l \mu_{ilc}  ) + \mu_{ilc}(B - \numobs \pi_i w_l  )}{n \pi_i w_l   + (B - n \pi_i w_l  )} \right|\\
    &\leq \frac{2 \numobs \tfac  +  2\numobs \delta_2}{\numobs \pi_i w_l - \numobs \sum_{c=1}^\kdim \tfac - \numobs \delta_2}
  \end{align*}
  Condition~\eqref{eqn:concentration-accuracy-condition} implies that $\sum_{c=1}^\kdim  \tfac \leq \mass \sum_{c=1}^\kdim \rho/8 \leq \mass/8$, where the last step follow from $\kdim \gap\leq 1$.
  Plugging this upper bound into the above inequality completes the proof.
\end{proof}
~\\

To proceed with the proof, we assign specific values to $\delta_1$ and $\delta_2$. Let
\begin{align}\label{eqn:delta-values}
  \delta_1 \defeq \min\left\{ \frac{\gap}{2},\frac{\gap\Dbar}{16}\right\} \quad \mbox{and} \quad \delta_2 \defeq \tmin/2.
\end{align}
We claim that at any step in the update, the preconditions~\eqref{eqn:assu-mu-concentration} and~\eqref{eqn:assu-q-concentration}
always hold.

We prove the claim by induction. Before the iteration begins, $\muhat$ is initialized such that the accuracy bound~\eqref{eqn:initialization-accuracy-condition}
holds. Thus, condition~\eqref{eqn:assu-mu-concentration} is satisfied at the beginning.
We assume by induction that condition~\eqref{eqn:assu-mu-concentration} is satisfied at time $1,2,\dots,\tau-1$
and condition~\eqref{eqn:assu-q-concentration} is satisfied at time $2,3,\dots,\tau-1$.
At time $\tau$, either update~\eqref{eqn:q_update} or update~\eqref{eqn:p_update} is performed.
If update~\eqref{eqn:q_update} is performed, then by the inductive hypothesis, condition~\eqref{eqn:assu-mu-concentration} holds before the update.
Thus, Lemma~\ref{lemma:q-concentration} implies that
\begin{align*}
  \max_{l\in[\kdim]}\{ |\qhat_{jl} - \indicator(y_j = l)|\} \leq \exp\left( -\mdim \left(\frac{\Dbar}{2} - \frac{2 \delta_1}{\gap - \delta_1}\right) + \log(\mdim) \right).
\end{align*}
The assignment~\eqref{eqn:delta-values} implies $\frac{\Dbar}{2} - \frac{2 \delta_1}{\gap - \delta_1} \geq \frac{\Dbar}{4}$, which yields that
\begin{align*}
  \max_{l\in[\kdim]}\{ |\qhat_{jl} - \indicator(y_j = l)|\} \leq \exp(-\mdim\Dbar/4 + \log(\mdim) ) \leq \tmin/2 = \delta_2,
\end{align*}
where the last inequality follows from condition~\eqref{eqn:concentration-accuracy-condition}. It suggests that condition~\eqref{eqn:assu-q-concentration}
holds after the update.

On the other hand, we assume that update~\eqref{eqn:p_update} is performed at time $\tau$. Since update~\eqref{eqn:p_update}
follows update~\eqref{eqn:q_update}, we have $\tau \geq 2$. By the inductive hypothesis, condition~\eqref{eqn:assu-q-concentration}
holds before the update, so Lemma~\ref{lemma:mu-concentration} implies
\begin{align*}
  |\muhat_{il} - \mu_{ilc}| \leq \frac{2 n \tfac +  2\numobs \delta_2}{(7/8)n \pi_i w_l -  \numobs \delta_2} =
  \frac{2 n \tfac +  \numobs \tmin}{(7/8)n \pi_i w_l -  \numobs \tmin/2} \leq
  \frac{3 \numobs \tfac }{(7/8)n \pi_i w_l -  \numobs \tmin/2},
\end{align*}
where the last step follows since $\tmin \leq \tfac $. Noticing $\gap\leq 1$, condition~\eqref{eqn:concentration-accuracy-condition} implies
that $\tmin \leq \mass /8$. Thus, the right hand side of the above inequality is bounded by $4\tfac/(\pi_i w_l)$. Using condition~\eqref{eqn:concentration-accuracy-condition} again,
we find
\[
    \frac{4\tfac}{\pi_i w_l} \leq \frac{4\tfac}{\mass} \leq \min\left\{ \frac{\gap}{2},\frac{\gap\Dbar}{16}\right\} = \delta_1,
\]
which verifies that condition~\eqref{eqn:assu-mu-concentration} holds after the update. This completes the induction.\\

Since preconditions~\eqref{eqn:assu-mu-concentration} and~\eqref{eqn:assu-q-concentration} hold for any time $\tau \geq 2$, Lemma~\ref{lemma:q-concentration}
and Lemma~\ref{lemma:mu-concentration} implies that the concentration bounds~\eqref{eqn:conclude-q-concentration} and~\eqref{eqn:conclude-mu-concentration}
always hold. These two concentration bounds establish the lemma's conclusion.

\subsection{Proof of Lemma~\ref{lemma:bound-variance}}

By the definition of $s_i$, we have
\begin{align*}
  \E[s_i^2] = \pi_i\sum_{c=1}^\kdim \mu_{iy_jc} (\log(\mu_{iy_jc}/\mu_{ilc}))^2 = \pi_i\sum_{c=1}^\kdim \mu_{iy_jc} (\log(\mu_{ilc}/\mu_{iy_jc}))^2
\end{align*}
We claim that for any $x \geq \gap$ and $\gap < 1$, the following inequality holds:
\begin{align}\label{eqn:log-related-inequality}
  \log^2(x) \leq \frac{2\log(1/\gap)}{1-\gap} (x - 1 - \log(x))
\end{align}
We defer the proof of inequality~\eqref{eqn:log-related-inequality}, focusing on its consequence. Let $x \defeq \mu_{ilc}/\mu_{iy_jc}$,
then inequality~\eqref{eqn:log-related-inequality} yields that
\begin{align*}
  \E[s_i^2]  \leq \frac{2\log(1/\gap)}{1-\gap} \pi_i \left(\sum_{c=1}^\kdim \mu_{ilc} - \mu_{iy_jc} - \mu_{iy_jc}\log(\mu_{ilc}/\mu_{iy_jc})\right) = \frac{2\log(1/\gap)}{1-\gap} \pi_i \KL{\mu_{iy_j}, \mu_{il}}.
\end{align*}

It remains to prove the claim~\eqref{eqn:log-related-inequality}. Let $f(x) \defeq \log^2(x) - \frac{2\log(1/\gap)}{1-\gap} (x - 1 - \log(x))$.
It suffices to show that $f(x)\leq 0$ for $x \geq \gap$. First, we have $f(1) = 0$ and
\begin{align*}
  f'(x) = \frac{2(\log(x) - \frac{\log(1/\gap)}{1-\gap}(x-1))}{x}.
\end{align*}
For any $x > 1$, we have
\[
    \log(x)< x-1\leq \frac{\log(1/\gap)}{1-\gap}(x-1)
\]
where the last inequality holds since $\log(1/\gap) \geq 1-\gap$. Hence, we have $f'(x) < 0$ and consequently $f(x) < 0$
for $x > 1$.

For any $\gap \leq x < 1$, notice that $\log(x) - \frac{\log(1/\gap)}{1-\gap}(x-1)$ is a concave function of $x$, and equals zero at two points
$x=1$ and $x=\gap$. Thus, $f'(x) \geq 0$ at any point $x\in[\gap,1)$, which implies $f(x)\leq 0$.

\section{Proof of Theorem~\ref{theorem:optimality}}
\label{sec:proof-optimality}

In this section we prove Theorem~\ref{theorem:optimality}. The proof separates into
two parts.

\subsection{Proof of Part (a)}

Throughout the proof, probabilities are implicitly conditioning on $\{\pi_i\}$ and $\{\mu_{ilc}\}$.
We assume that $(l,l')$ are the pair of labels such that
\begin{align*}
  \Dbar = \frac{1}{m} \sum_{i=1}^\mdim \pi_i \KL{\mu_{il},\mu_{il'}}.
\end{align*}
Let $\Qprob$ be a uniform distribution over the set $\{l,l'\}^\numobs$. For any predictor $\yhat$, we have
\begin{align}
    \max_{v\in [\kdim]^\numobs} \E\Big[\sum_{j=1}^\numobs \indicator(\yhat_j\neq y_j) \Big | y=v\Big] &\geq \sum_{v\in\{l,l'\}^\numobs}\Qprob(v)\;\E\Big[\sum_{j=1}^\numobs \indicator(\yhat_j\neq y_j) \Big | y=v\Big] \nonumber \\
    &= \sum_{j=1}^\numobs \sum_{v\in\{l,l'\}^\numobs} \Qprob(v)\;\E\Big[ \indicator(\yhat_j\neq y_j) \Big | y=v\Big]. \label{eqn:y-lower-extreme-to-uniform-1}
\end{align}
Thus, it is sufficient to lower bound the right hand side of inequality~\eqref{eqn:y-lower-extreme-to-uniform-1}.

For the rest of the proof, we lower bound the quantity $\sum_{y\in\{l,l'\}^\numobs} \Qprob(v)\;\E[ \indicator(\yhat_j\neq y_j)  | y]$
for every item $j$. Let $Z\defeq \{z_{ij}: i\in[\mdim],~j\in[\numobs]\}$ be the set of all observations. We define two probability measures
$\Pprob_0$ and $\Pprob_1$, such that $\Pprob_0$ is the measure of $Z$ conditioning on $y_j = l$, while $\Pprob_1$ is the measure of
$Z$ conditioning on $y_j = l'$. By applying Le Cam's method~\cite{yu1997assouad} and Pinsker's inequality, we have
\begin{align}
  \sum_{v\in\{l,l'\}^\numobs} \Qprob(v)\;\E\Big[ \indicator(\yhat_j\neq y_j) \Big | y=v\Big] & = \Qprob(y_j=l) \Pprob_0(\yhat_j \neq l) + \Qprob(y_j=l') \Pprob_1(\yhat_j \neq l')  \nonumber \\ & \geq \frac{1}{2}- \frac{1}{2}\tv{\Pprob_0 - \Pprob_1}\nonumber\\
  &\geq \frac{1}{2} - \frac{1}{4}\sqrt{\KL{\Pprob_0, \Pprob_1}}.\label{eqn:y-lower-pinsker}
\end{align}
The remaining arguments upper bound the KL-divergence between $\Pprob_0$ and $\Pprob_1$. Conditioning on $y_j$, the set of random variables
$Z_j\defeq \{z_{ij}: i\in[\mdim]\}$ are independent of $Z\backslash Z_j$ for both $\Pprob_0$ and $\Pprob_1$. Letting the distribution of $X$ with respect to
probability measure $\Pprob$ be denoted by $\Pprob(X)$, we have
\begin{align}\label{eqn:y-lower-row-independence}
  \KL{\Pprob_0, \Pprob_1} = \KL{\Pprob_0(Z_j), \Pprob_1(Z_j)} + \KL{\Pprob_0(Z\backslash Z_j), \Pprob_1(Z\backslash Z_j)} = \KL{\Pprob_0(Z_j), \Pprob_1(Z_j)},
\end{align}
where the last step follows since $\Pprob_0(Z\backslash Z_j) = \Pprob_1(Z\backslash Z_j)$. Next, we observe that $z_{1j},z_{2j},\dots,z_{\mdim j}$
are mutually independent given $y_j$, which implies
\begin{align}\label{eqn:y-lower-column-independence}
    \KL{\Pprob_0(Z_j), \Pprob_1(Z_j)} &= \sum_{i=1}^\mdim \KL{\Pprob_0(z_{ij}), \Pprob_1(z_{ij})} \nonumber\\
    &= (1-\pi_i) \log\left(\frac{1-\pi_i}{1-\pi_i}\right) + \sum_{c=1}^\kdim \pi_i\mu_{ilc} \log\left( \frac{\pi_i\mu_{ilc}}{\pi_i\mu_{il'c}} \right) \nonumber\\
    &=  \sum_{c=1}^\kdim \pi_i \KL{\mu_{ilc}, \mu_{il'c}} = \mdim \Dbar.
\end{align}
Combining inequality~\eqref{eqn:y-lower-pinsker} with equations~\eqref{eqn:y-lower-row-independence} and~\eqref{eqn:y-lower-column-independence},
we have
\begin{align*}
  \sum_{v\in\{l,l'\}^\numobs} \Qprob(v)\;\E\Big[ \indicator(\yhat_j\neq y_j) \Big | y=v\Big] \geq \frac{1}{2}- \frac{1}{4}\sqrt{\mdim \Dbar}.
\end{align*}
Thus, if $\mdim \leq 1/(4\Dbar)$, then the above inequality is lower bounded by $3/8$. Plugging this lower bound into
inequality~\eqref{eqn:y-lower-extreme-to-uniform-1} completes the proof.

\subsection{Proof of Part (b)}

Throughout the proof, probabilities are implicitly conditioning on $\{\pi_i\}$ and $\{w_l\}$. We define
two vectors
\begin{align*}
  u_0 \defeq \left(\frac{1}{2},\frac{1}{2},0,\dots,0\right)^T \in \R^\kdim \qquad \mbox{and} \qquad u_1 \defeq \left(\frac{1}{2}+\delta,\frac{1}{2}-\delta,0,\dots,0\right)^T \in \R^\kdim
\end{align*}
where $\delta \leq 1/4$ is a scalar to be specified. Consider a $\mdim$-by-$\kdim$ random matrix $V$ whose entries are uniformly sampled
from $\{0,1\}$. We define a random tensor $u_{V}\in \R^{\mdim\times\kdim\times\kdim}$, such that $(u_{V})_{il}\defeq u_{V_{il}}$ for all $(i,l)\in [\mdim]\times [\kdim]$.
Givan an estimator $\muhat$ and a pair of indices~$(\itilde,\ltilde )$, we have
\begin{align}\label{eqn:y-lower-extreme-to-uniform}
     \sup_{\mu\in \R^{\mdim\times\kdim\times\kdim}} \E\Big[ \ltwos{\muhat_{\itilde \ltilde } - \mu_{\itilde \ltilde }}^2 \Big] \geq \sum_{v\in [\kdim]^\numobs} \mprob(y=v) \left( \sum_{V} \mprob(V)\; \E\Big[ \ltwos{\muhat_{\itilde \ltilde } - \mu_{\itilde \ltilde }}^2 \Big| \mu = u_{V},y=v\Big] \right).
\end{align}

For the rest of the proof, we lower bound the term $\sum_{V} \mprob(V)\; \E[ \ltwos{\muhat_{\itilde \ltilde } - \mu_{\itilde \ltilde }}^2 | \mu = u_{V},y=v]$
for every $v\in [\kdim]^\numobs$. Let $\Vhat$ be an estimator defined as
\begin{align*}
  \Vhat = \left\{ \begin{array}{ll}
    0 & \mbox{if $\ltwos{\muhat_{\itilde \ltilde } - u_0} \leq \ltwos{\muhat_{\itilde \ltilde } - u_1}$}.\\
    1 & \mbox{otherwise}.
  \end{array}
  \right.
\end{align*}
If $\mu = u_{V}$, then $\Vhat \neq V_{\itilde \ltilde } \Rightarrow \ltwos{\muhat_{\itilde \ltilde } - \mu_{\itilde \ltilde }} \geq \frac{\sqrt{2}}{2}\delta$.
Consequently, we have
\begin{align}\label{eqn:hypothesis-testing}
  \sum_{V} \mprob(V)\; \E[ \ltwos{\muhat_{\itilde \ltilde } - \mu_{\itilde \ltilde }}^2 | \mu = u_{V},y=v] \geq \frac{\delta^2}{2} \mathrm{P}[\Vhat \neq V_{\itilde \ltilde } | y = v].
\end{align}

Let $Z\defeq \{z_{ij}: i\in[\mdim],~j\in[\numobs]\}$ be the set of all observations. We define two probability measures
$\Pprob_0$ and $\Pprob_1$, such that $\Pprob_0$ is the measure of $Z$ conditioning on $y = v$ and $\mu_{\itilde \ltilde } = u_0$, and
$\Pprob_1$ is the measure of $Z$ conditioning on $y = v$ and $\mu_{\itilde \ltilde } = u_1$. For any other pair of indices
$(i,l)\neq (\itilde ,\ltilde )$, $\mu_{il} = u_{V_{il}}$ for both $\Pprob_0$ and $\Pprob_1$.
By this definition, the distribution of $Z$ conditioning on $y=v$ and $\mu = u_{V}$ is a mixture of distributions
$\Qprob\defeq \frac{1}{2}\Pprob_0 + \frac{1}{2}\Pprob_1$.
By applying Le Cam's method~\cite{yu1997assouad} and Pinsker's inequality, we have
\begin{align}
  \mathrm{P}[\Vhat \neq V_{\itilde\ltilde} | y = v] &\geq \frac{1}{2}- \frac{1}{2}\tv{\Pprob_0 - \Pprob_1}\nonumber\\
  &\geq \frac{1}{2} - \frac{1}{4}\sqrt{\KL{\Pprob_0, \Pprob_1}}.\label{eqn:mu-lower-pinsker}
\end{align}
Conditioning on $y=v$, the set of random variables $Z_i\defeq \{z_{ij}:j\in[\numobs]\}$ are mutually independent for both $\Pprob_0$ and $\Pprob_1$.
Letting the distribution of $X$ with respect to probability
measure $\Pprob$ be denoted by $\Pprob(X)$, we have
\begin{align}\label{eqn:mu-lower-indepenence-1}
  \KL{\Pprob_0, \Pprob_1} = \sum_{i=1}^\mdim \KL{\Pprob_0(Z_i), \Pprob_1(Z_i)} = \KL{\Pprob_0(Z_{\itilde }), \Pprob_1(Z_{\itilde })}
\end{align}
where the last step follows since $\Pprob_0(Z_{i}) =  \Pprob_1(Z_{i})$ for all $i\neq \itilde $. Next, we let $J \defeq \{j: v_j = \ltilde \}$ and define a set of random variables $Z_{iJ}\defeq \{z_{ij}:j\in J\}$. It is straightforward
to see that $Z_{iJ}$ is independent of $Z_i \backslash Z_{iJ}$ for both $\Pprob_0$ and $\Pprob_1$. Hence, we have
\begin{align}
  \KL{\Pprob_0(Z_{\itilde }), \Pprob_1(Z_{\itilde })} &= \KL{\Pprob_0(Z_{\itilde J}), \Pprob_1(Z_{\itilde J})} + \KL{\Pprob_0(Z_{\itilde }\backslash Z_{\itilde J}), \Pprob_1(Z_{\itilde }\backslash  Z_{\itilde J})} \nonumber\\
  &= \KL{\Pprob_0(Z_{\itilde J}), \Pprob_1(Z_{\itilde J})}\label{eqn:mu-lower-indepenence-2}
\end{align}
where the last step follows since $\Pprob_0(Z_{\itilde }\backslash Z_{\itilde J}) =  \Pprob_1(Z_{\itilde }\backslash Z_{\itilde J})$.
Finally, since $\mu_{\itilde\ltilde}$ is explicitly given in both $\Pprob_0$ and $\Pprob_1$, the random variables contained in $Z_{\itilde J}$ are mutually independent.
Consequently, we have
\begin{align}
  \KL{\Pprob_0(Z_{\itilde J}), \Pprob_1(Z_{\itilde J})} &= \sum_{j\in J} \KL{\Pprob_0(z_{\itilde j}), \Pprob_1(z_{\itilde j})} = |J|\;\pi_{\itilde}\; \frac{1}{2}\log\left(\frac{1}{1-4\delta^2}\right) \nonumber\\
  &\leq \frac{5}{2} |J|\;\pi_{\itilde}\delta^2.\label{eqn:mu-lower-indepenence-3}
\end{align}
Here, we have used the fact that $\log(1/(1-4x^2)) \leq 5x^2$ holds for any $x\in [0,1/4]$. \\

Combining the lower bound~\eqref{eqn:mu-lower-pinsker} with upper bounds~\eqref{eqn:mu-lower-indepenence-1},~\eqref{eqn:mu-lower-indepenence-2}
and~\eqref{eqn:mu-lower-indepenence-3}, we find
\begin{align*}
  \Pprob[\Vhat_{il} \neq V_{il} | y = v] \geq \frac{3}{8} \indicator\left( \frac{5}{2} |J|\;\pi_{\itilde}\delta^2 \leq \frac{1}{4}\right).
\end{align*}
Plugging the above lower bound into inequalities~\eqref{eqn:y-lower-extreme-to-uniform} and~\eqref{eqn:hypothesis-testing} implies that
\begin{align*}
  \sup_{\mu\in \R^{\mdim\times\kdim\times\kdim}} \E\Big[ \ltwos{\muhat_{\itilde \ltilde } - \mu_{\itilde \ltilde }}^2 \Big] \geq \frac{3\delta^2}{16} \mprob\left[ |\{j:y_j=\ltilde\}| \leq \frac{1}{10\pi_{\itilde}\delta^2} \right].
\end{align*}
Note than $|\{j:y_j=\ltilde\}| \sim \mbox{Binomial}(\numobs,w_{\ltilde})$. Thus, if we set
\begin{align*}
  \delta^2 \defeq \min\left\{\frac{1}{16}, \frac{1}{10\pi_{\itilde} w_{\ltilde}\numobs}\right\},
\end{align*}
then $\frac{1}{10\pi_{\itilde}\delta^2}$ is greater than or equal to the median of $|\{j:y_j=\ltilde\}|$, and consequently,
\begin{align*}
  \sup_{\mu\in \R^{\mdim\times\kdim\times\kdim}} \E\Big[ \ltwos{\muhat_{\itilde \ltilde } - \mu_{\itilde \ltilde }}^2 \Big] \geq \min\left\{\frac{3}{512}, \frac{3}{320\pi_{\itilde} w_{\ltilde}\numobs}\right\},
\end{align*}
which establishes the theorem.


\section{Proof of Theorem~\ref{theorem:one-coin-model}}
\label{sec:proof-one-coin-model}

Our proof strategy is briefly described as follow: We first upper bound the error of Step (1)-(2) in Algorithm~\ref{alg:one-coin-model}.
This upper bound is presented as lemma~\ref{lemma:onecoin-init}.
Then, we analyze the performance of Step (3), taking the guarantee obtained from the previous two steps.

\begin{lemma}\label{lemma:onecoin-init}
Assume that $\dgap_3 > 0$. Let $\phat_i$ be initialized by Step (1)-(2). For any scalar $0 < t < \frac{\dgapbar \dgap_3^3}{18}$,
the upper bound
\begin{align}
 \max_{i\in[m]}\{|\phat_i - p_i|\} \leq \frac{18 t}{\dgap_3^3}
\end{align}
holds with probability at least $1-m^2\exp(-nt^2/2)$.
\end{lemma}

\noindent The rest of the proof upper bounds the error of Step (3). The proof follows very similarly steps
as in the proof of Theorem~\ref{theorem:stage-2}. We first
define two events that will be shown holding with high probability.
\begin{align*}
  &\event_1:~\sum_{i=1}^\mdim \sum_{c=1}^\kdim \indicator(z_{ij} = e_c)\log(\mu_{iy_jc}/\mu_{ilc}) \geq \mdim \Dbar/2 \qquad \mbox{for all $j\in[\numobs]$ and $l\in [\kdim]\backslash\{y_j\}$}.\\
  &\event_2:~\Big| \sum_{j=1}^\numobs \indicator(z_{ij} = e_{y_j}) - \numobs p_i \Big| \leq \numobs t_i \qquad \mbox{for all $i\in [\mdim]$}.
\end{align*}

\begin{lemma}\label{lemma:onecoin-update-stability}
  Assume that $\event_1\cap \event_2$ holds. Also assume that $\gap \leq p_i \leq 1-\gap$ for all $i\in [\mdim]$. If $\phat$ is initialized such that
  \begin{align}\label{eqn:onecoin-initialization-accuracy-condition}
    |\phat_i - p_i| \leq \minac \defeq \min\left\{ \frac{\dgapbar}{2}, \frac{\gap}{2},\frac{\gap\Dbar}{16}\right\}\qquad \mbox{for all $i\in [\mdim]$}
  \end{align}
  and scalars $t_i$ satisfy
  \begin{align}\label{eqn:onecoin-concentration-accuracy-condition}
     \exp\Big( -\mdim\Dbar/4 + \log(\mdim) \Big) \leq t_i \leq \min\left\{ \frac{\gap}{4},\frac{\gap \Dbar}{32}\right\}
  \end{align}
  Then the estimates $\phat$ and $\qhat$ obtained by alternating updates~\eqref{eqn:onecoin-q_update} and~\eqref{eqn:onecoin-p_update}
  satisfy:
  \begin{align*}
    |\phat_i - p_i| &\leq 2 t_i. \qquad &&\mbox{for all $i\in[\mdim]$}.\\
    \max_{l\in[\kdim]}\{ |\qhat_{jl} - \indicator(y_j = l)|\}  &\leq \exp\left( -\mdim\Dbar/4 + \log(\mdim) \right) \qquad &&\mbox{for all $j\in [\numobs]$}.
  \end{align*}
\end{lemma}

As in the proof of Theorem~\ref{theorem:stage-2}, we can lower bound the probability of the event $\event_1\cap\event_2$
by applying Bernstein's inequality and the Chernoff bound. In particular, the following bound holds:
\begin{align}\label{eqn:onecoin-good-event-probability-bound}
  \mprob[\event_1\cap\event_2] \geq 1 - \kdim \numobs \exp\left(-\frac{\mdim \Dbar}{33 \log(1/\gap)} \right) - \sum_{i=1}^\mdim 2\exp\left( - \frac{\numobs t_i^2 }{3 p_i}\right).
\end{align}
The proof of inequality~\eqref{eqn:onecoin-good-event-probability-bound} precisely follows the proof of Theorem~\ref{theorem:stage-2}.

\paragraph{Proof of upper bounds (a) and (b) in Theorem~\ref{theorem:one-coin-model}} To apply Lemma~\ref{lemma:onecoin-update-stability}, we need to ensure that condition~\eqref{eqn:onecoin-initialization-accuracy-condition}
 holds. If we assign $t\defeq \minac \dgap_3^3/18$ in Lemma~\ref{lemma:onecoin-init}, then condition~\eqref{eqn:onecoin-initialization-accuracy-condition}
holds with probability at least $1 - \mdim^2 \exp(- \numobs \minac^2 \dgap_3^6 /648)$. To ensure that this event holds with probability at least
$1 - \delta/3$, we need to have
\begin{align}\label{eqn:onecoin-initialization-success-condition}
  \numobs \geq \frac{648 \log(3\mdim^2/\delta)}{\minac^2 \dgap_3^6}.
\end{align}
By Lemma~\ref{lemma:onecoin-update-stability}, for $\yhat_j = y_j$ being true, it suffices to have
\begin{align}\label{eqn:onecoin-classification-mn-requirement-1}
    \mdim > 4\log(2m)/{\Dbar}
\end{align}
To ensure that $\event_1\cap\event_2$ holds with probability at least $1-2\delta/3$, expression~\eqref{eqn:onecoin-good-event-probability-bound}
needs to be lower bounded by $1 - 2\delta/3$. It is achieved by
\begin{align}\label{eqn:onecoin-classification-mn-requirement-2}
  \mdim \geq \frac{ 33 \log(1/\gap)\log(3\kdim\numobs/\delta) }{\Dbar} \quad \mbox{and} \quad \numobs \geq \frac{3 p_i \log(6\mdim/\delta)}{t_i^2}
\end{align}
If we choose
\begin{align}\label{eqn:onecoin-regression-assign-til}
    t_i \defeq \sqrt{ \frac{3 \log(6\mdim/\delta)}{\numobs} }.
\end{align}
then the second part of condition~\eqref{eqn:onecoin-classification-mn-requirement-2} is guaranteed.
To ensure that $\tfac$ satisfies condition~\eqref{eqn:onecoin-concentration-accuracy-condition}. We need to have
\begin{align*}
  \sqrt{ \frac{3 \log(6\mdim/\delta)}{\numobs} } &\geq \exp\Big( -\mdim\Dbar/4 + \log(\mdim) \Big) \quad\mbox{and}\\
  \sqrt{ \frac{3 \log(6\mdim/\delta)}{\numobs} } &\leq \minac/2.
\end{align*}
The above two conditions  requires that $\mdim$ and $\numobs$ satisfy
\begin{align}
  \mdim &\geq \frac{4\log(\mdim\sqrt{\numobs/(3\log(6\mdim/\delta))})}{\Dbar} \label{eqn:onecoin-regression-mn-requirement-1}\\
  \numobs &\geq \frac{12\log(6\mdim/\delta)}{\minac^2} \label{eqn:onecoin-regression-mn-requirement-2}
\end{align}
The five conditions~\eqref{eqn:onecoin-initialization-success-condition},~\eqref{eqn:onecoin-classification-mn-requirement-1},~\eqref{eqn:onecoin-classification-mn-requirement-2},~\eqref{eqn:onecoin-regression-mn-requirement-1}
and~\eqref{eqn:onecoin-regression-mn-requirement-2} are simultaneously satisfied if we have
\begin{align*}
  \mdim &\geq \frac{\max\{ 33 \log(1/\gap)\log(3\kdim\numobs/\delta), 4\log(2\mdim \numobs)  \}}{\Dbar}\quad\mbox{and}\\
  \numobs &\geq \frac{648\log(3\mdim^2/\delta)}{\minac^2 \dgap_3^6}.
\end{align*}
Under this setup, $\yhat_j = y_j$ holds for all $j\in [\numobs]$ with probability at least $1-\delta$. Combining
equation~\eqref{eqn:onecoin-regression-assign-til}
with Lemma~\ref{lemma:onecoin-update-stability}, the bound
\begin{align*}
  |\phat_i - p_i| \leq 2\sqrt{\frac{3\log(6\mdim/\delta)}{\numobs}}
\end{align*}
holds with probability at least $1-\delta$.

\subsection{Proof of Lemma~\ref{lemma:onecoin-init}}

We claim that after initializing $\phat$ via formula~\eqref{eqn:onecoin-p-init}, it satisfies
\begin{align}\label{eqn:claim-two-side-concentration}
 \min\left\{ \max_{i\in[m]}\{|\phat_i - p_i|\}, \max_{i\in[m]}\{|\phat_i - (2/\kdim-p_i)|\} \right\} \leq \frac{18 t}{\dgap_3^3}
\end{align}
with probability at least $1-m^2\exp(-nt^2/2)$. Assuming inequality~\eqref{eqn:claim-two-side-concentration}, it is straightforward to see that this bound
is preserved by the algorithm's step (2). In addition, step (2) ensures that
$\frac{1}{\mdim} \sum_{i=1}^\mdim \phat_i \geq \frac{1}{\kdim}$, which implies
\begin{align}\label{eqn:onecoin-side-is-correct}
  \max_{i\in[m]}\{|\phat_i - (2/\kdim-p_i)|\} &\geq \left| \frac{1}{\mdim} \sum_{i=1}^\mdim \phat_i - \left( \frac{2}{\kdim} -  \frac{1}{\mdim} \sum_{i=1}^\mdim p_i \right)\right| \nonumber\\
  & \geq \frac{1}{\kdim} - \left(\frac{1}{\kdim} - \dgapbar \right) = \dgapbar > \frac{18 t}{\dgap_3^3}.
\end{align}
Combining inequalities~\eqref{eqn:claim-two-side-concentration} and~\eqref{eqn:onecoin-side-is-correct} establishes
the lemma.\\

We turn to prove claim~\eqref{eqn:claim-two-side-concentration}.
For any worker $a$ and worker $b$, it is obvious that $\indicator(z_{aj} = z_{bj})$ are independent random variables for
$j=1,2,\dots,\numobs$. Since
\begin{align*}
  \E[\indicator(z_{aj} = z_{bj})] = p_a p_b + (\kdim - 1)\; \frac{1-p_a}{\kdim - 1}\;\frac{1-p_b}{\kdim - 1} = \frac{k}{\kdim - 1}\;(p_a - 1/\kdim)(p_b - 1/\kdim) + \frac{1}{\kdim}
\end{align*}
and $\frac{\kdim - 1}{\kdim}(\indicator(z_{aj} = z_{bj}) - \frac{1}{\kdim})$ belongs to the interval $[-1,1]$, applying Hoeffding's inequality implies that
\begin{align*}
  \mprob( | N_{ab} - (p_a - 1/\kdim)(p_b - 1/\kdim)| \leq t ) \geq 1-\exp(-\numobs t^2/2) \quad \mbox{for any $t> 0$}.
\end{align*}
By applying the union bound, the inequality
\begin{align}\label{eqn:covariance-entry-error-bound}
    | N_{a b} - (p_a - 1/\kdim)(p_b - 1/\kdim)| \leq t
\end{align}
holds for all $(a,b)\in [\mdim]^2$ with probability at least $1 - \mdim^2\exp(-\numobs t^2/2)$. For the rest of the proof, we assume that this high-probability event holds.\\

Given an arbitrary index $i$, we take indices $(a_i,b_i)$ such that
\begin{align}\label{eqn:take-greatest-pair}
  (a_i,b_i) = \arg\max_{(a,b)} \{ |N_{ab}|: a\neq b\neq i\}.
\end{align}
We consider another two indices $(\astar, \bstar)$ such that $|p_{\astar}-1/\kdim|$ and $|p_{\bstar}-1/\kdim|$ are the two greatest
elements in $\{|p_a-1/\kdim|:a\in[m]\backslash \{i\}\}$. Let $\beta_i\defeq p_i - 1/\kdim$
be a shorthand notation, then inequality~\eqref{eqn:covariance-entry-error-bound} and equation~\eqref{eqn:take-greatest-pair} yields that
\begin{align}\label{eqn:ajak-bound}
  |\beta_{a_i} \beta_{b_i}| \geq |N_{a_i b_i}| - t \geq |N_{\astar \bstar}| - t \geq |\beta_{\astar}\beta_{\bstar}| -2t \geq |\beta_{\astar}\beta_{\bstar}|/2
\end{align}
where the last step follows since $2t\leq \dgap_3^2/2 \leq |\beta_{\astar}\beta_{\bstar}|/2$.
Note that $|\beta_{b_i}| \leq |\beta_{\astar}|$ (since $|\beta_{\astar}|$ is the largest entry by its definition), inequality~\eqref{eqn:ajak-bound} implies that
$|\beta_{a_i}| \geq \frac{|\beta_{\bstar}\beta_{\astar}|}{2|\beta_{b_i}|} \geq \frac{|\beta_{\bstar}|}{2} \geq \dgap_3/2$. By  the same argument, we obtain $|\beta_{b_i}| \geq |\beta_{\bstar}|/2 \geq \dgap_3/2$.
To upper bound the estimation error, we write $|N_{i a_i}|$, $|N_{i b_i}|$, $|N_{a_i b_i}|$ in the form of
\begin{align*}
  |N_{i a_i}| &= |\beta_{i} \beta_{a_i}| + \delta_1\\
  |N_{i b_i}| &= |\beta_i \beta_{b_i}| + \delta_2\\
  |N_{a_i b_i}| &= |\beta_{a_i} \beta_{b_i}| + \delta_3
\end{align*}
where $|\delta_1|,|\delta_2|,|\delta_3| \leq t$. Firstly, notice that $N_{i a_i},N_{i b_i}\in [-1,1]$, thus,
\begin{align}\label{eqn:gap-bound-1}
  \left| \sqrt{\frac{|N_{i a_i}N_{i b_i}|}{|N_{a_i b_i}|}} -\sqrt{\frac{|N_{i a_i}N_{i b_i}|}{|\beta_{a_i} \beta_{b_i}|}} \right| \leq \left| \frac{1}{\sqrt{|N_{a_i b_i}|}} - \frac{1}{\sqrt{|\beta_{a_i} \beta_{b_i}|}} \right| \leq \frac{t}{2(|\beta_{a_i} \beta_{b_i}|-t)^{3/2}} \leq \frac{t}{(\dgap_3^2/4)^{3/2}}
\end{align}
where the last step relies on the inequality $|\beta_{a_i} \beta_{b_i}| - t \geq \dgap_3^2/4$ obtained by inequality~\eqref{eqn:ajak-bound}.
Secondly, we upper bound the difference between $\sqrt{|N_{i a_i}N_{i b_i}|}$ and $\sqrt{|\beta_i^2\beta_{a_i} \beta_{b_i}|}$.
If $|\beta_i| \leq t$, using the fact that $|\beta_{a_i}|,|\beta_{b_i}| \leq 1$, we have
\begin{align*}
  \left|\sqrt{|N_{i a_i}N_{i b_i}|} - \sqrt{|\beta_i^2 \beta_{a_i} \beta_{b_i}|}\right| \leq \sqrt{|N_{i a_i}N_{i b_i}|} + \sqrt{|\beta_i^2 \beta_{a_i} \beta_{b_i}|} \leq \sqrt{4t^2}+ \sqrt{t^2} \leq 3t
\end{align*}
If $|\beta_i| > t$, using the fact that $|\beta_{a_i}|,|\beta_{b_i}|\in [\dgap_3/2,1]$ and $|\beta_{a_i} \beta_{b_i}| \geq \dgap_3^2/2$, we have
\begin{align*}
  \left|\sqrt{|N_{i a_i}N_{i b_i}|} - \sqrt{|\beta_i^2 \beta_{a_i} \beta_{b_i}|}\right| &\leq \frac{|\beta_i \beta_{b_i}\delta_1| + |\beta_i \beta_{a_i}\delta_2| + |\delta_1\delta_2|}{\sqrt{|\beta_i^2 \beta_{a_i} \beta_{b_i}|}}\\
  &\leq \frac{|\delta_1|}{\sqrt{|\beta_{a_i}/\beta_{b_i}|}} + \frac{|\delta_2|}{\sqrt{|\beta_{b_i}/\beta_{a_i}|}} + \frac{|\delta_1\delta_2|}{t\sqrt{|\beta_{a_i} \beta_{b_i}|}}\\
  &\leq 3\sqrt{2}t/\dgap_3.
\end{align*}
Combining the above two upper bounds implies
\begin{align}\label{eqn:gap-bound-2}
  \left|\sqrt{\frac{|N_{i a_i}N_{i b_i}|}{|\beta_{a_i} \beta_{b_i}|}} -\sqrt{\frac{|\beta_i^2 \beta_{a_i} \beta_{b_i}|}{|\beta_{a_i} \beta_{b_i}|}}  \right| = \frac{\left|\sqrt{|N_{i a_i}N_{i b_i}|} -\sqrt{|\beta_i^2 \beta_{a_i} \beta_{b_i}|}  \right|}{\sqrt{|\beta_{a_i} \beta_{b_i}|}} \leq \frac{6t}{\dgap_3^2}.
\end{align}
Combining inequalities~\eqref{eqn:gap-bound-1} and~\eqref{eqn:gap-bound-2}, we obtain
\begin{align}\label{eqn:gap-bound-3}
  \left| \sqrt{\frac{|N_{i a_i}N_{i b_i}|}{|N_{a_i b_i}|}} - |\beta_i| \right| \leq \frac{14t}{\dgap_3^3}.
\end{align}
\\

Finally, we turn to analyzing the sign of $N_{i a_1}$. According to inequality~\eqref{eqn:covariance-entry-error-bound}, we have
\begin{align*}
  N_{i a_1} = \beta_i \beta_{a_1} + \delta_4
\end{align*}
where $|\delta_4| \leq t$. Following the same argument for $\beta_{a_i}$ and $\beta_{b_i}$, it was shown that $|\beta_{a_1}| \geq \dgap_3/2$.
We combine inequality~\eqref{eqn:gap-bound-3} with a case study of $\sign(N_{i a_1})$ to complete the proof. Let
\begin{align*}
  \phat_i \defeq \frac{1}{\kdim} + \sign(N_{i a_1})\sqrt{\frac{|N_{i a_i}N_{i b_i}|}{|N_{a_i b_i}|}}.
\end{align*}
If $\sign(N_{i a_1}) \neq \sign(\beta_i \beta_{a_1})$, then
$|\beta_i \beta_{a_1}|\leq |\delta_4| \leq t$. Thus, $|\beta_i| \leq t/|\beta_{a_i}| \leq 2t/\dgap_3$, and consequently,
\begin{align}\label{eqn:estimation-error-bound-inter-1}
  \max\{ | \phat_i - p_i |, | \phat_i - (2/\kdim - p_i) |\} &\leq \left| \sqrt{\frac{|N_{i a_i}N_{i b_i}|}{|N_{a_i b_i}|}} \right| + |p_i - 1/\kdim| \nonumber\\
  &\leq \left| \sqrt{\frac{|N_{i a_i}N_{i b_i}|}{|N_{a_i b_i}|}} - |p_i - 1/\kdim|\right| + 2|p_i -1/\kdim|\leq \frac{18 t}{\dgap_3^3}
\end{align}
Otherwise, we have $\sign(N_{i a_1}) = \sign(\beta_i \beta_{a_1})$ and consequently
$\sign(\beta_i) = \sign(N_{i a_1})\sign(\beta_{a_1})$. If $\sign(\beta_{a_1}) = 1$, then $\sign(\beta_i) = \sign(N_{i a_1})$, which yields that
\begin{align}\label{eqn:estimation-error-bound-inter-2}
  | \phat_i - p_i | = \left| \sign(N_{i a_1})\sqrt{\frac{|N_{i a_i}N_{i b_i}|}{|N_{a_i b_i}|}} - \sign(\beta_i)|\beta_i| \right| \leq \frac{14t}{\dgap_3^3}.
\end{align}
If $\sign(\beta_{a_1}) = -1$, then $\sign(\beta_i) = -\sign(N_{i a_1})$, which yields that
\begin{align}\label{eqn:estimation-error-bound-inter-3}
  | \phat_i - (2/\kdim-p_i) | = \left| \sign(N_{i a_1})\sqrt{\frac{|N_{i a_i}N_{i b_i}|}{|N_{a_i b_i}|}} + \sign(\beta_i)|\beta_i| \right| \leq \frac{14t}{\dgap_3^3}.
\end{align}
Combining inequalities~\eqref{eqn:estimation-error-bound-inter-1},~\eqref{eqn:estimation-error-bound-inter-2} and~\eqref{eqn:estimation-error-bound-inter-3},
we find that
\begin{align*}
  \min\left\{ \max_{i\in[m]}\{|\phat_i - p_i|\}, \max_{i\in[m]}\{|\phat_i - (2/\kdim-p_i)|\} \right\}\leq \frac{18t}{\dgap_3^3}.
\end{align*}
which establishes claim~\eqref{eqn:claim-two-side-concentration}.

\subsection{Proof of Lemma~\ref{lemma:onecoin-update-stability}}

The proof follows the argument in the proof of Lemma~\ref{lemma:update-stability}. We present two
lemmas upper bounding the error of update~\eqref{eqn:onecoin-q_update} and update~\eqref{eqn:onecoin-p_update},
assuming proper initialization.

\begin{lemma}\label{lemma:onecoin-q-concentration}
  Assume that event $\event_1$ holds. If $p$ and its estimate $\phat$ satisfies
  \begin{align}\label{eqn:onecoin-assu-mu-concentration}
    \gap \leq p_i \leq 1-\gap \quad\mbox{and}\quad |\phat_i - p_i | \leq \delta_1 \qquad \mbox{for all $i\in[\mdim]$},
  \end{align}
  and $\qhat$ is updated by formula~\eqref{eqn:onecoin-q_update}, then $\qhat$ is bounded as:
  \begin{align}\label{eqn:onecoin-conclude-q-concentration}
    \max_{l\in[\kdim]}\{ |\qhat_{jl} - \indicator(y_j = l)|\} & \leq \exp\left( -\mdim \left(\frac{\Dbar}{2} - \frac{2 \delta_1}{\gap - \delta_1}\right) + \log(\mdim) \right) \qquad \mbox{for all $j\in [\numobs]$}.
  \end{align}
\end{lemma}

\begin{proof}
Following the proof of Lemma~\ref{lemma:q-concentration}, the lemma is established since both $|\log(\phat_i/p_i)|$
and $|\log((1-\phat_i)/(1-p_i))|$ are bounded by $\log(\gap/(\gap - \delta_1))$.
\end{proof}

\begin{lemma}\label{lemma:onecoin-p-concentration}
  Assume that event $\event_2$ holds. If $\qhat$ satisfies
  \begin{align}\label{eqn:onecoin-assu-q-concentration}
    \max_{l\in [\kdim]}\{|\qhat_{jl} - \indicator(y_j=l)|\} \leq \delta_2 \qquad \mbox{for all $j\in [\numobs]$},
  \end{align}
  and $\phat$ is updated by formula~\eqref{eqn:onecoin-p_update}, then $\phat$ is bounded as:
  \begin{align}\label{eqn:onecoin-conclude-mu-concentration}
    |\phat_i - p_i| \leq t_i + \delta_2. \qquad \mbox{for all $i\in[\mdim]$}.
  \end{align}
\end{lemma}

\begin{proof}
  By formula~\eqref{eqn:onecoin-p_update}, we have
  \begin{align*}
    \phat_i - p_i = \frac{1}{\numobs} \left( \sum_{j=1}^\numobs \indicator(z_{ij} = e_{y_i}) -\numobs p_i\right) + \frac{1}{\numobs} \sum_{j=1}^\numobs \sum_{l=1}^\kdim  ( \qhat_{il} - \indicator(y_j = l))\indicator(z_{ij} = e_{l}).
  \end{align*}
  Combining inequality~\eqref{eqn:onecoin-assu-q-concentration} with the inequality implied by event $\event_2$  completes the proof.
\end{proof}

Following the steps in the proof of Lemma~\ref{lemma:update-stability}, we assign specific values to $\delta_1$ and $\delta_2$. Let
\begin{align*}
  \delta_1 \defeq \min\left\{ \frac{\gap}{2},\frac{\gap\Dbar}{16}\right\} \quad \mbox{and} \quad \delta_2 \defeq \min_{i\in[\mdim]}\{t_i\}.
\end{align*}
By the same inductive argument for proving Lemma~\ref{lemma:update-stability}, we can show that the upper bounds~\eqref{eqn:onecoin-conclude-q-concentration}
and~\eqref{eqn:onecoin-conclude-mu-concentration} always hold after the first iteration. Plugging the assignments of $\delta_1$ and $\delta_2$
into upper bounds~\eqref{eqn:onecoin-conclude-q-concentration}
and~\eqref{eqn:onecoin-conclude-mu-concentration} completes the proof.


\section{Basic Lemmas}

In this section, we prove some standard lemmas that we use for proving technical results.

\begin{lemma}[Matrix Inversion]\label{lemma:matrix-inversion}
Let $A,E\in \R^{k\times k}$ be given, where $A$ is invertible and $E$ satisfies that $\lops{E}\leq \sigma_k(A)/2$. Then
\begin{align*}
  \lops{(A+E)^{-1} - A^{-1}} \leq \frac{2\lops{E}}{\sigma_k^2(A)}.
\end{align*}
\end{lemma}

\begin{proof}
A little bit of algebra reveals that
\begin{align*}
(A+E)^{-1} - A^{-1} = (A+E)^{-1} E A^{-1}.
\end{align*}
Thus, we have
\begin{align*}
\lops{(A+E)^{-1} - A^{-1}} \leq \frac{\lops{E}}{\sigma_k(A)\sigma_k(A+E)}
\end{align*}
We can lower bound the eigenvalues of $A+E$ by $\sigma_k(A)$ and $\lops{E}$. More concretely, since
\begin{align*}
  \ltwos{(A+E)\theta} \geq  \ltwos{A\theta} - \ltwos{E\theta} \geq \sigma_k(A) - \lops{E}
\end{align*}
holds for any $\ltwos{\theta} = 1$, we have $\sigma_k(A+E) \geq \sigma_k(A) - \lops{E}$. By the assumption that $\lops{E}\leq \sigma_k(A)/2$, we have $\sigma_k(A+E) \geq \sigma_k(A)/2$.
Then the desired bound follows.
\end{proof}

\begin{lemma}[Matrix Multiplication]\label{lemma:matrix-multiplication}
Let $A_i,E_i\in \R^{k\times k}$ be given for $i=1,\dots,n$, where the matrix $A_i$ and the perturbation matrix $E_i$ satisfy $\lops{A_i}\leq K_i,~ \lops{E_i}\leq K_i $.
Then
\begin{align*}
  \lop{\prod_{i=1}^n (A_i+E_i)- \prod_{i=1}^n A_i} \leq 2^{n-1}\left( \sum_{i=1}^n \frac{\lops{E_i}}{K_i}\right) \prod_{i=1}^n K_i
\end{align*}
\end{lemma}

\begin{proof}
 By triangular inequality, we have
 \begin{align*}
   \lop{\prod_{i=1}^n (A_i+E_i)- \prod_{i=1}^n A_i} &=  \lop{\sum_{i=1}^n \left(\prod_{j=1}^{i-1}A_j\right)\left(\prod_{k=i+1}^{n} (A_k+E_k) \right)E_i } \\
   & \leq  \sum_{i=1}^n \lops{E_i} \left(\prod_{j=1}^{i-1}\lops{A_j}\right)\left(\prod_{k=i+1}^{n} \lops{A_k+E_k} \right) \\
   & \leq \sum_{i=1}^n 2^{n-i} \frac{\lops{E_i}}{K_i} \prod_{i=1}^n K_i\\
   & = 2^{n-1}\left( \sum_{i=1}^n \frac{\lops{E_i}}{K_i}\right) \prod_{i=1}^n K_i
 \end{align*}
 which completes the proof.
\end{proof}

\begin{lemma}[Matrix and Tensor Concentration]\label{lemma:tensor-concentration-lemma}
Let $\{X_j\}_{j=1}^n$, $\{Y_j\}_{j=1}^n$ and $\{Z_j\}_{j=1}^n$ be i.i.k.~samples from some distribution over $\R^k$
with bounded support ($\ltwos{X}\leq 1$, $\ltwos{Y}\leq 1$ and $\ltwos{Z}\leq 1$ with probability 1). Then with probability at least $1-\delta$,
\begin{align}
  \norm{\frac{1}{n}\sum_{j=1}^n X_j\otimes Y_j - \E[X_1\otimes Y_1]}_F &\leq \frac{1 + \sqrt{\log(1/\delta)}}{\sqrt{n}}.\label{eqn:matrix-concentration-lemma}\\
  \norm{\frac{1}{n}\sum_{j=1}^n X_j\otimes Y_j\otimes Z_j - \E[X_1\otimes Y_1\otimes Z_1]}_F &\leq \frac{1 + \sqrt{\log(k/\delta)}}{\sqrt{n/k}}.\label{eqn:tensor-concentration-lemma}
\end{align}
\end{lemma}

\begin{proof}
Inequality~\eqref{eqn:matrix-concentration-lemma} is proved in Lemma D.1 of~\cite{anandkumar2012spectral}. To prove inequality~\eqref{eqn:tensor-concentration-lemma},
we note that for any tensor $T\in \R^{k\times k\times k}$, we can define $k$-by-$k$ matrices $T_1,\dots,T_k$ such that $(T_i)_{jk} \defeq T_{ijk}$. As a result,
we have $\norm{T}_F^2 = \sum_{i=1}^k \norm{T_i}_F^2$. If we set $T$ to be the tensor on the left hand side of inequality~\eqref{eqn:tensor-concentration-lemma},
then
\begin{align*}
  T_i = \frac{1}{n}\sum_{j=1}^n (Z_j^{(i)}\;X_j)\otimes Y_j - \E[(Z_j^{(i)}\;X_1)\otimes Y_1]
\end{align*}
By applying the result of inequality~\eqref{eqn:matrix-concentration-lemma},
we find that with probability at least $1-k\delta'$, we have
\begin{align*}
  \norm{\frac{1}{n}\sum_{j=1}^n X_j\otimes Y_j\otimes Z_j - \E[X_1\otimes Y_1\otimes Z_1]}_F^2 \leq k\left(\frac{1 + \sqrt{\log(1/\delta')}}{\sqrt{n}}\right)^2.
\end{align*}
Setting $\delta' = \delta/k$ completes the proof.
\end{proof}


\newpage
\bibliographystyle{abbrv}
\bibliography{bib}

\end{document}